\documentclass[11pt]{article}

\usepackage[margin=1in]{geometry}

\usepackage[utf8]{inputenc} 
\usepackage[T1]{fontenc}    
\usepackage{microtype}
\usepackage{graphicx}
\usepackage{subfigure}
\usepackage{multicol}
\usepackage{booktabs} 
\usepackage{multirow}
\usepackage{amssymb,amsmath,amsthm,amsfonts}
\usepackage{tikz}
\usetikzlibrary{calc}
\usepackage{bm}
\usepackage{textgreek}
\usepackage{mathtools}
\usepackage{enumitem}
\usepackage{authblk}
\usepackage{graphicx}
\usepackage[font=small]{caption}
\usepackage[linesnumbered,ruled,vlined,boxed,algo2e]{algorithm2e}
\usepackage{physics}
\usepackage{footnote}
\usepackage{xcolor}
\usepackage{mathrsfs}
\usepackage{bbm}
\usepackage{makecell}
\usepackage[ruled,vlined,linesnumbered]{algorithm2e}
\usepackage{algorithmic}
\usepackage{algorithm}

\usepackage[pagebackref]{hyperref}
\hypersetup{colorlinks=true,urlcolor=blue,linkcolor=blue,citecolor=[rgb]{.0,.50,.32},}

\usepackage{verbatim}

\usepackage[numbers,sort]{natbib}

\newtheorem{theorem}{Theorem}
\newtheorem{definition}{Definition}
\newtheorem{lemma}{Lemma}

\newtheorem{corollary}{Corollary}

\newtheorem{assumption}{Assumption}

\newtheorem{problem}{Problem}

\newcommand{\eq}[1]{(\ref{eq:#1})}

\newcommand{\thm}[1]{\hyperref[thm:#1]{Theorem~\ref*{thm:#1}}}
\newcommand{\cor}[1]{\hyperref[cor:#1]{Corollary~\ref*{cor:#1}}}
\newcommand{\defn}[1]{\hyperref[defn:#1]{Definition~\ref*{defn:#1}}}
\newcommand{\lem}[1]{\hyperref[lem:#1]{Lemma~\ref*{lem:#1}}}
\newcommand{\prop}[1]{\hyperref[prop:#1]{Proposition~\ref*{prop:#1}}}
\newcommand{\assum}[1]{\hyperref[assum:#1]{Assumption~\ref*{assum:#1}}}
\newcommand{\fig}[1]{\hyperref[fig:#1]{Figure~\ref*{fig:#1}}}
\newcommand{\tab}[1]{\hyperref[tab:#1]{Table~\ref*{tab:#1}}}
\newcommand{\algo}[1]{\hyperref[algo:#1]{Algorithm~\ref*{algo:#1}}}
\renewcommand{\sec}[1]{\hyperref[sec:#1]{Section~\ref*{sec:#1}}}
\newcommand{\append}[1]{\hyperref[append:#1]{Appendix~\ref*{append:#1}}}
\newcommand{\fac}[1]{\hyperref[fac:#1]{Fact~\ref*{fac:#1}}}
\newcommand{\lin}[1]{\hyperref[lin:#1]{Line~\ref*{lin:#1}}}
\newcommand{\prob}[1]{\hyperref[prob:#1]{Problem~\ref*{prob:#1}}}

\def\>{\rangle}
\def\<{\langle}

\newcommand{\inner}[2]{\left\langle #1,\, #2 \right\rangle}
\newcommand{\vect}[1]{\ensuremath{\mathbf{#1}}}
\newcommand{\x}{\ensuremath{\mathbf{x}}}
\newcommand{\y}{\ensuremath{\mathbf{y}}}
\newcommand{\z}{\ensuremath{\mathbf{z}}}
\newcommand{\g}{\ensuremath{\mathbf{g}}}
\newcommand{\e}{\ensuremath{\mathbf{e}}}
\newcommand{\h}{\ensuremath{\mathbf{h}}}
\newcommand{\w}{\ensuremath{\mathbf{w}}}
\newcommand{\p}{\ensuremath{\mathbf{p}}}

\newcommand{\R}{\mathbb{R}}

\newcommand{\E}{\mathbb{E}}

\DeclareMathOperator{\comp}{comp}
\DeclareMathOperator{\diag}{diag}

\renewcommand{\d}{\mathrm{d}}
\renewcommand{\x}{\vect{x}}
\renewcommand{\u}{\vect{u}}
\renewcommand{\v}{\vect{v}}
\newcommand{\0}{\mathbf{0}}

\def\:{\hbox{\bf:}}

\SetKwInput{KwInput}{Input} 
\SetKwInput{KwParameter}{Parameters}           
\SetKwInput{KwOutput}{Output}              
\SetKwInput{KwReturn}{Return}
\SetKwComment{Comment}{/* }{ */}
\SetKwBlock{RepeatBlock}{repeat}{end repeat}

\let\oldnl\nl
\newcommand{\nonl}{\renewcommand{\nl}{\let\nl\oldnl}}

\usepackage[textsize=tiny]{todonotes}

\title{Finding Stationary Points by Comparisons}

\date{}

\author{
  Helin Wang$^{1,3}$\thanks{Equal contribution. } \qquad Chenyi Zhang$^{2,}$\protect\footnotemark[1]\qquad Xiwen Tao$^{1,3}$\qquad Yexin Zhang$^{3,4}$\qquad Tongyang Li$^{3,4,}$\thanks{Corresponding author. Email: tongyangli@pku.edu.cn} \\
  $^1$ School of Electronics Engineering and Computer Science, Peking University\\
  $^2$ Computer Science Department, Stanford University \\
  $^3$ Center on Frontiers of Computing Studies, Peking University\\
  $^4$ School of Computer Science, Peking University\\
}

\begin{document}

\maketitle
\begin{abstract}
We study the problem of finding stationary points of non-convex functions when access to the objective is provided only through a comparison oracle that, given two points, outputs which has the larger function value. For a twice differentiable $f\colon\mathbb R^n\to\mathbb R$ with Lipschitz gradient and Hessian, we develop an algorithm that visits an $\epsilon$-stationary point using $\widetilde O(n^2/\epsilon^{1.5})$ queries. Our approach uses a subroutine that estimates the normalized Hessian to accuracy $\delta$ using $\widetilde O(n^2\log(1/\delta))$ queries. We further study this problem with a quantum comparison oracle model where queries can be made in superpositions, and develop the first quantum algorithm that finds an $\epsilon$-stationary point, which takes $\widetilde O(n/\epsilon^{1.5})$ queries.

\end{abstract}


\section{Introduction}
We study the problem of finding a stationary point of a non-convex function $f\colon\R^n\to\R$, which is a point $\x\in\R^n$ satisfying $\|\nabla f(\x)\|\leq \epsilon$, given an initial $\x^{(0)}\in\R^n$ with bounded initial function value gap, i.e., $f(\x^{(0)})-\inf_{x\in\R^n}f(\x)\leq \Delta$. Finding a stationary point is a natural objective in non-convex optimization, since computing the global optimum is NP hard in the worst case. Moreover, in problems including tensor decomposition~\cite{ge2017optimization}, matrix completion~\cite{ge2016matrix}, and regression with non-convex regularization~\cite{loh2015regularized}, global optimality can be obtained by finding a second-order stationary point, which is a generalization of stationary points. 

As a foundational problem in optimization theory, finding a stationary point, also known as critical point computation~\cite{agarwal2017finding,adil2025balancing} or making the gradient small~\cite{nesterov2012make,allen2018make}, has been extensively studied. Given a gradient oracle, gradient descent finds an $\epsilon$-stationary point of $f$ using $O(1/\epsilon^2)$ iterations, provided that $f$ has Lipschitz gradient, i.e., $\|\nabla f(\x)-\nabla f(\y)\|\leq L_1\|\x-\y\|$ for any $\x,\y\in\R^n$~\cite{nesterov2013introductory}. More generally, if $f$ has $L_p$-Lipschitz $p$-th order derivative, Birgin et al.~\cite{birgin2017worst} gives an algorithm using $O(L_p^{1/p}\Delta\epsilon^{-(p+1)/p})$ queries to a $p$-th order oracle, where a query at $\x$ returns all derivatives of $f$ at $\x$ up to order $p$. Furthermore, Carmon et al.~\cite{carmon2020lower} establishes that these rates are optimal among dimension-independent algorithms. 

More recently, training of machine learning models solicits for even simpler information. For example, it is known that taking only signs of gradient descents still demonstrates good performance in training neural networks~\cite{liu2018signsgd,li2023faster,bernstein2018signsgd}. Moreover, in the breakthrough of large language models (LLMs), reinforcement learning from human feedback (RLHF)~\cite{christiano2017deep,gaur2024global,fan2024fedrlhf} played an important role in training these LLMs, for instance on GPT-3~\cite{ouyang2022training}. Compared to standard RL that applies function evaluation for rewards, RLHF is preference-based RL that only compares between options and determines which is better. There is emerging interest in preference-based RL, with a series of results~\cite{chen2022human,saha2023dueling,novoseller2020dueling,xu2020preference,zhu2023principled,tang2023zeroth} establishing provable guarantees for learning a near-optimal policy from preference feedback. Furthermore, Wang et al.~\cite{wang2023rlhf} proved that preference-based RL can be solved with small or no extra costs compared to those of standard reward-based RL for a wide range of models.

More broadly, these developments reflect a longstanding interest in designing optimization methods using limited feedback. There has been a line of research on solving optimization problems using a comparison oracle; see the survey~\cite{larson2019derivative}. Formally, a function $f\colon\mathbb R^n\to\mathbb R$ is accessed through a comparison oracle $O_f^{\comp}\colon\R^n\times\R^n\rightarrow\{-1,1\}$ that upon a pair of inputs $(\x,\y)\in\R^n\times\R^n$, we get output:
\begin{equation}\label{eq:comparison}
O_f^{\comp}(\x,\y)=
    \begin{cases}
    1, & f(\x)\ge f(\y) \\
    -1. & f(\x)\le f(\y)
    \end{cases}
\end{equation}
with either output allowed when $f(\mathbf x)=f(\mathbf y)$. Classical direct-search and pattern-search methods use such comparisons to accept or reject trial steps~\cite{kolda2003optimization,audet2006mesh}, while the Nelder--Mead method~\cite{nelder1965simplex} relies on comparisons among candidate points to drive the search. However, Nelder--Mead may converge to a nonstationary point even on continuously differentiable convex examples~\cite{mckinnon1998convergence}.
Based on that, Jamieson et al.~\cite{jamieson2012query} studied derivative-free optimization with Boolean comparison feedback and comparison-based line searches. GradientLess Descent~\cite{golovin2020gradientless} studies monotone-invariant zeroth-order optimization, and~\cite{tang2023zeroth} studied ranking-based feedback, where the oracle provides ranking information over multiple candidate points. Beyond noiseless pairwise comparisons in Euclidean space, Saha et al.~\cite{saha2024faster} studied batched and multiway preference feedback, and Ren et al.~\cite{ren2026riemannianduelingoptimization} considered comparison-based optimization on Riemannian manifolds. Very recently, Ref.~\cite{scheinberg2026functionfreeoptimizationcomparisonoracles} proposed a function-free optimization framework in which the preference relation itself, rather than an underlying scalar objective, defines the optimization problem.

For convex objectives, several works give provable guarantees under comparison or order-type feedback. Karabag et al.~\cite{karabag2021smooth} proposed an ellipsoid-based method for smooth convex optimization using comparison oracles. Bergou et al.~\cite{bergou2020stochastic} and Gorbunov et al.~\cite{gorbunov2019stochastic} developed the stochastic three-point (STP) methods using comparisons among randomly sampled trial points to select the next iterate. Dueling optimization~\cite{saha2021dueling,saha2022dueling} studies convex optimization from noisy pairwise comparisons, and Lobanov et al.~\cite{lobanov2024acceleration} developed accelerated methods in a similar setting. 

For nonconvex objectives, fewer comparison-query guarantees are known. The stochastic three-point method of~\cite{bergou2020stochastic,gorbunov2019stochastic} samples random search directions and chooses the next iterate using comparisons among trial points; for functions with Lipschitz gradients, this gives an $O(n/\epsilon^2)$ comparison-query implementation for finding an $\epsilon$-stationary point in expectation. These results leave open whether comparison access can exploit higher-order smoothness to improve the dependence on $\epsilon$ for finding stationary points.

\paragraph{Our results.}
In this paper, we develop an algorithm that makes $\widetilde O(n^2/\epsilon^{1.5})$ queries to $O_f^{\mathrm{comp}}$ and guarantees that one of the queried points is an $\epsilon$-stationary point.\footnote{We use $\widetilde O(\cdot)$ to hide poly-logarithmic factors in $n$, $\epsilon^{-1}$, $L_1$, $L_2$, and $\Delta$.}

\begin{theorem}[Informal]\label{thm:Comparison-Newton-informal}
Let $f\colon\R^n\to\R$ have $L_1$-Lipschitz gradient and $L_2$-Lipschitz Hessian. Given $\x_0\in\R^n$ satisfying $f(\x_0)-\inf_{x\in\R^n}(\x)\leq \Delta$, with success probability at least $2/3$, \algo{best-of-all} visits an $\epsilon$-stationary point using $\widetilde O(\Delta\sqrt{L_2}n^2/\epsilon^{1.5})$ queries to $O_f^{\mathrm{comp}}$.
\end{theorem}

Up to poly-logarithmic factors, the $\epsilon$-dependence of our rate matches the optimal rate of second-order methods~\cite{nesterov2006cubic,carmon2020lower}. Compared to~\cite{bergou2020stochastic,gorbunov2019stochastic}, our algorithm achieves an improved dependence on $\epsilon$, at the cost of a worse dependence on the dimension $n$. It is unclear whether our bound, particularly the $n^2$ dependence, is asymptotically optimal in any non-trivial regime of $n$ and $\epsilon$. Although~\cite{carmon2020lower} give relevant lower bounds in the dimension-independent setting, the dimension-dependent complexity of finding stationary points is still not well understood, even in the setting where the algorithm has access to gradient or Hessian information. We view developing matching lower bounds in the comparison setting as an interesting open problem.

Note that here we only obtain the guarantee that our algorithm visits an $\epsilon$-stationary point throughout the iterations instead of finding one, similar to prior works~\cite{bergou2020stochastic,gorbunov2019stochastic}. This is due to the fact that in the comparison oracle model the algorithm only observes relative function values. Consequently, it is in general impossible to access the gradient norm or even to test whether a given point $\x$ is an $\epsilon$-stationary point. For applications where a single point is needed, a natural heuristic is to return the iterate with the smallest function value, which can be identified by comparisons. While this does not certify stationarity in general, in many structured nonconvex problems near-optimality is closely tied to stationarity or second-order stationarity, such as tensor decomposition~\cite{ge2015decomposing}, matrix completion~\cite{ge2016matrix}, and regression with nonconvex regularization~\cite{loh2015regularized}.

Furthermore, we study the problem of finding stationary points in the quantum setting, where we can query a quantum comparison oracle
\begin{equation}\label{eq:quantum_comparison}
O_{f,q}^{\comp}|\x\>|\y\>|b\>=|\x\>|\y\>|b\oplus \mathbbm{1}\{f(\x)>f(\y)\}\>
\end{equation}
which performs comparison in quantum superpositions.\footnote{Quantum notations can be found in~\sec{prelim}.} Extending upon our classical algorithm, we obtain a quantum algorithm for finding stationary points using the quantum comparison oracle \eqref{eq:quantum_comparison}.

\begin{theorem}[Informal]\label{thm:quantum-main-intro}
Let $f\colon\R^n\to\R$ have $L_1$-Lipschitz gradient and $L_2$-Lipschitz Hessian. Given $\x_0\in\R^n$ satisfying $f(\x_0)-\inf_{x\in\R^n}(\x)\leq \Delta$, there exists a quantum algorithm that visits an $\epsilon$-stationary point using $\widetilde O(\Delta\sqrt{L_2}n/\epsilon^{1.5})$ queries to $O_{f,q}^{\mathrm{comp}}$.
\end{theorem}

\paragraph{Techniques.}

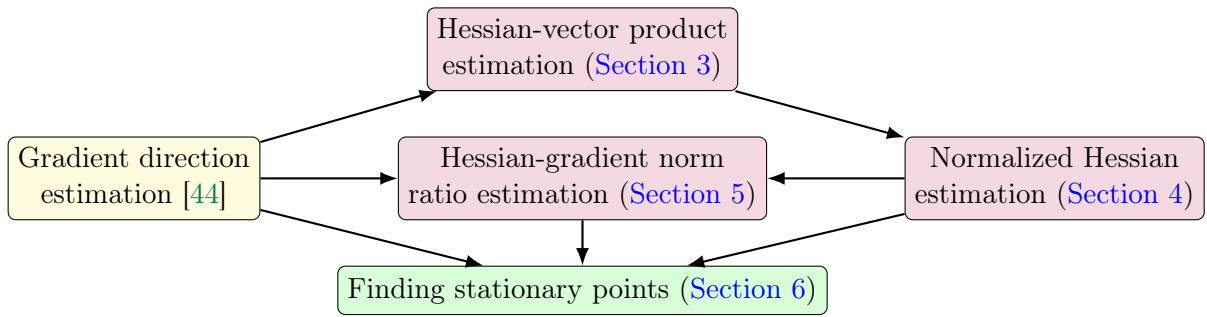
\begin{figure*}[htbp!]
\centering
\caption{The structure of our stationary point finding algorithm. In this figure the arrow represents the implementation relationship. The three purple frames represent the novel techniques that we propose, the yellow frame represents the known technique from~\cite{tao2026comparison}, and the green frame represents our main theorem.}
\begin{tikzpicture}[
  node distance=7mm and 7mm,
  box/.style={draw, rounded corners=1mm, minimum width=21mm, minimum height=6mm, align=center},
  arr/.style={-Latex, thick},
  darr/.style={-Latex, thick, dashed}
]
  \usetikzlibrary{arrows.meta,positioning}

  \node[box, fill=purple!15] (hv) {Hessian-vector product \\ estimation (\sec{hess_vec_product})};
  
  \node[box, fill=yellow!15, below left=6mm and 22mm of hv] (grad) {Gradient direction \\ estimation \cite{tao2026comparison}};
  
  \node[box, fill=purple!15, below right=6mm and 22mm of hv] (nh) {Normalized Hessian\\ estimation (\sec{hess_est})};
  
  \node[box, fill=purple!15, below=6mm of hv] (hgn) {Hessian-gradient norm\\ ratio estimation (\sec{grad_hess_ratio})};

  \node[box, fill=green!15, below=6mm of hgn, minimum width=15mm] (stat) {Finding stationary points (\sec{best_of_all})};

  \draw[arr] (grad) -- (hv);
  \draw[arr] (grad) -- (hgn);
  \draw[arr] (grad) -- (stat);

  \draw[arr] (hv) -- (nh);
  \draw[arr] (nh) -- (hgn);

  \draw[arr] (nh) -- (stat);
  \draw[arr] (hgn) -- (stat);
\end{tikzpicture}
\label{fig:flowchart}
\end{figure*}

Our approach for establishing \thm{Comparison-Newton-informal} is summarized in \fig{flowchart}. Throughout, we use the comparison-based gradient direction estimation algorithm as a basic primitive.

We first give an algorithm \texttt{ComparisonHessVec} (\algo{Comparison-Triangle}) that, for any given $\x,\y\in\R^n$, estimates the direction $\nabla^2 f(\x)\cdot \y$. This is done by first estimating the directions of $\nabla f(\x)$, $\nabla f(\x+r\y)$, and $\nabla f(\x-r\y)$ using~\cite{tao2026comparison}. We then infer the direction of $\nabla f(\x+r\y)-\nabla f(\x-r\y)$ using a geometric property: its intersection with $\nabla f(\x)$ and $\nabla f(\x+r\y)$ as well as its intersection with $\nabla f(\x)$ and $\nabla f(\x-r\y)$ give two segments of same length (see \fig{triangle-Hessian-vector-product-estimation}). By choosing $r$ sufficiently small, this direction provides a good approximation of the direction of $\nabla^2 f(\x)\cdot \y$.

\begin{figure}[ht]
\centering
\caption{The intuition of \texttt{ComparisonHessVec} (\algo{Comparison-Triangle}) for computing Hessian-vector products using gradient directions.}
\begin{tikzpicture}[scale=1.1]
  \coordinate (O) at (0,0);
  
  \draw[->] (O) -- (120:3cm) node[at end, below right] {};
  \draw[->] (O) -- (45:3cm) node[at end, above right] {};
  \draw[->] (O) -- (0,3) node[at end, above] {};

  \draw (45:1.5cm) -- (120:2.25cm);
  \draw (-0.5,1.58) -- (-0.5,1.78);
  \draw (-0.4,1.58) -- (-0.4,1.78);
  \draw (0.5,1.22) -- (0.5,1.42);
  \draw (0.4,1.22) -- (0.4,1.42);
  \draw[fill] (45:1.5cm) circle (2pt) node[above right] {};
  \draw[fill] (120:2.25cm) circle (2pt) node[below right] {};
  \draw[fill] (O) circle (2pt) node[below left] {\small $\x$};
  
  \node at (3,2) {$\frac{\nabla f(\x+r\y)}{\|\nabla f(\x+r\y)\|}$};
  \node at (-2,3) {$\frac{\nabla f(\x-r\y)}{\|\nabla f(\x-r\y)\|}$};
  \node at (0,3.5) {$\frac{\nabla f(\x)}{\|\nabla f(\x)\|}$};
  \node at (2.6,0.7) {\footnotesize direction of $\nabla^2 f(\x) \cdot \y$};
\end{tikzpicture}
\label{fig:triangle-Hessian-vector-product-estimation}
\end{figure}
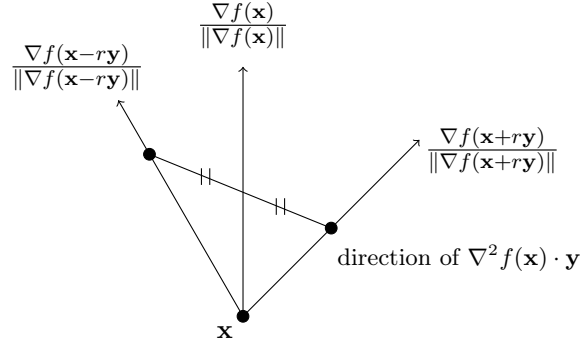

Then, we use \texttt{ComparisonHessVec} to implement \texttt{ComparisonHE} (\algo{normhess}), which estimates the normalized Hessian $\nabla^2 f(\x)/\|\nabla^2 f(\x)\|$. Because a column of $\nabla^2f(\x)$ is given by multiplying a unit vector by  $\nabla^2f(\x)$, we can use \texttt{ComparisonHessVec} to estimate the normalized direction of each column. The remaining challenge is to determine the relative norms of different columns. For any two columns $\h_i$ and $\h_j$ of $\nabla^2f(\x)$, this is accomplished by estimating the directions of $\h_i$, $\h_j$, and $\h_i+\h_j$, and then solving the resulting triangle. A corner case of the above argument arises when two columns $\h_i$ and $\h_j$ have nearly identical directions, in which case the error incurred by solving the triangle can be arbitrarily large. To address this issue, we introduce an additional perturbation along a direction sufficiently far from that of $\h_i$, which ensures that the resulting error remains controllable.

Furthermore, we give an algorithm \texttt{ComparisonRatio} (\algo{Comparison-ratio}) that estimates the ratio between the norm of the Hessian and the norm of the gradient using comparisons. This is done by identifying a vector $\v$ of the approximate Hessian whose corresponding eigenvalue has the largest magnitude, estimating the directions of $\nabla f(\x)$ and $\nabla f(\x+\mu \v)$, and then solving the resulting triangle. One source of error in this approximation is that $\v$ may differ from $\u$, the eigenvector of the true Hessian $\nabla^2 f(\x)$ corresponding to the eigenvalue with the largest magnitude. We control this discrepancy by applying the Davis–Kahan theorem to bound $\|\u-\v\|$, which yields a bound on the overall estimation error.

Finally, we design a trust-region–type algorithm (\algo{best-of-all}) such that, at any iteration $\x_t$ which is not an $\epsilon$-stationary point, it produces a new point $\x_{t+1}$ that decreases the function value by $\Omega(\epsilon^{3/2})$. As shown in standard trust-region analysis~\cite{sorensen1982newton}, there exists a point $\hat\x\in\R^n$ with $\|\hat\x-\x_t\|=O(\sqrt{\epsilon})$ such that $f(\hat\x)-f(\x_t)\le -\Omega(\epsilon^{3/2})$. Moreover, such a point can be found exactly given access to the normalized gradient, the normalized Hessian, and the ratio of their norms. We show that the approximations of these quantities produced by our previous algorithms are sufficient to recover a point $\hat\x'$ with the same asymptotic decrease in function value. Hence, after running this algorithm for $\Theta(1/\epsilon^{1.5})$ iterations, it decreases the function value for at least $\Omega(\Delta)$ in expectation, thus we get a contradiction with a constant probability. This contradiction means that we have visited at least one stationary point during the iteration steps. The above argument can fail in certain corner cases, e.g., when the Hessian is rank-one, positive semidefinite, or negative semidefinite, or when the eigenvector $v$ corresponding to the eigenvalue of largest magnitude is nearly aligned with the gradient. In these cases, we additionally apply normalized gradient descent together with a binary line search to guarantee an $\Omega(\epsilon^{-1.5})$ decrease in function value.

\paragraph{Open questions.}
Our work leaves several open questions for future investigation. 
\begin{itemize}
    \item First, is it possible to interpolate between our approach and the STP framework~\cite{bergou2020stochastic,gorbunov2019stochastic} to achieve a sharper tradeoff between the dependencies on $n$ and $\epsilon$?
    \item Second, it is natural to extend our results to stochastic comparison oracles, including the standard Bradley--Terry model~\cite{bradley1952rank}.
    \item Third, it is worth investigating the practical applicability of our algorithm in practical settings.
    \item Fourth, our approach may benefit from lazy Hessian-update techniques~\cite{doikov2025first,liu2026quantum}. In \append{lazy_TR}, we show that a lazy trust-region method using $\widetilde O(n^2+n^{3/2}/\epsilon^{3/2})$ queries to a zeroth-order oracle. Establishing an analogous result in the comparison-oracle model is a natural direction for future work.
\end{itemize}


\section{Preliminaries}\label{sec:prelim}
\paragraph{Basic notation.}
For any $\mathbf{x}\in\mathbb{R}^n$, we denote $\g$ and $H$ as the gradient and Hessian of $f$ at $\mathbf{x}$, respectively, and omit the dependence on $\mathbf{x}$ when it is clear from context. For any vector $\v$ and matrix $A$, let $\|\v\|$ and $\|A\|$ be their $\ell_2$ norm. We denote $B_R^{n}(\x) \coloneqq \{\y\in\R^n\colon\|\y-\x\|\leq R\}$ and $S_R^{n}(\x) \coloneqq \{\y\in\R^n\colon\|\y-\x\|=R\}$. 

\paragraph{Quantum notation.}
We briefly introduce some basic notation of quantum computing that will be used in this paper. We use the Dirac notation $|\cdot\rangle$ to represent quantum states, which can be seen as column vectors. Qubit, which is the basic unit of quantum state, is represented as $|\phi\rangle = \alpha |0\rangle + \beta |1\rangle$ on the computational basis $\{|0\rangle, |1\rangle\}$ with $\alpha, \beta \in \mathbb{C}$ and $|\alpha|^2 + |\beta|^2 = 1$. An $n$-qubit system is represented in a $2^n$-dimensional vector space, with basis states $\{|0\rangle, |1\rangle\}^{\otimes n}$. An $n$-qubit state can be written as $|\psi\rangle = \sum_{x \in \{0,1\}^n} \alpha_x |x\rangle$ with $\sum_{x} |\alpha_x|^2 = 1$. A quantum algorithm can access a function via queries to a quantum oracle. For a classical function $\phi$, the oracle $O_\phi$ is defined as a unitary transformation that maps $|x\rangle|b\rangle$ to $|x\rangle|b \oplus \phi(x)\rangle$.
This allows the oracle to be queried on superpositions of inputs, producing corresponding superpositions of outputs.

\paragraph{Gradient direction estimation using comparisons.}

Recently, Tao et al.~\cite{tao2026comparison} developed classical and quantum algorithms for gradient direction estimation using comparisons. These algorithms serve as basic building blocks in our algorithms, stated below:

\begin{theorem}[Classical gradient direction estimation, Theorem 3.3 in~\cite{tao2026comparison}]\label{thm:Comparison-GDE}
Let $f\colon\mathbb R^n\to\mathbb R$ with $L_1$-Lipschitz gradient. Given query access to a comparison oracle $\mathcal O_f^{\comp}$ \eq{comparison} and any precision $\epsilon$, there exists an algorithm $\texttt{ComparisonGE}(\x,\epsilon,\gamma)$ such that, for any $\x\in\mathbb R^n$ with $\|\nabla f(\x)\|\ge\gamma$, it outputs a unit vector $\widehat \g$ satisfying $\left\|\widehat \g-\frac{\nabla f(\x)}{\|\nabla f(\x)\|}\right\|\le\epsilon$ with success probability at least $2/3$, using $O(n\log(1/\epsilon))$ queries.
\end{theorem}

\begin{theorem}[Quantum gradient direction estimation, Theorem 4.2 in~\cite{tao2026comparison}]\label{thm:Comparison-QGDE}
In the setting of \thm{Comparison-GDE}, given any precision $\epsilon$, there exists an algorithm $\texttt{ComparisonQGE}(\x,\epsilon,\gamma)$ such that, for any $\x\in\mathbb R^n$ with $\|\nabla f(\x)\|\ge\gamma$, it outputs a unit vector $\widehat \g$ satisfying $\left\|\widehat \g-\frac{\nabla f(\x)}{\|\nabla f(\x)\|}\right\|\le\epsilon$ with success probability at least $8/15-2\epsilon$, using $O(\log (n/\epsilon))$ queries.
\end{theorem}


\section{Hessian-Vector Product Estimation by Comparisons}\label{sec:hess_vec_product}

In this section, we introduce $\texttt{ComparisonHessVec}$ (\algo{Comparison-Triangle}), an algorithm for approximating the direction of a Hessian–vector product using comparisons.

\begin{algorithm2e}[htbp!]
	\caption{\texttt{ComparisonHessVec}}
	\label{algo:Comparison-Triangle}
	\LinesNumbered
	\DontPrintSemicolon
    \KwIn{$f\colon\R^n\to\R$, $\x,\y\in\R^n$, precision $\hat{\delta}$, $\gamma_\x$, $\gamma_\y$.}
    
    Set $r_0\leftarrow \min\left\{\frac{\gamma_{\x}}{100L_1},\frac{\gamma_{\x}}{100L_2},\frac{\sqrt{\gamma_{\x}\hat{\delta}}}{20\sqrt{L_2}},\frac{\gamma_{\y}\hat{\delta}\sqrt{\epsilon}}{20\sqrt{L_2}}\right\}$.\;
    
    $\hat{\g}_{0}\leftarrow\texttt{ComparisonGE}(\x,\frac{L_2 r_{0}^{2}}{\gamma_{\x}},\gamma_\x)$, $\hat{\g}_{1}\leftarrow\texttt{ComparisonGE}(\x+r_{0}\y,\frac{\rho r_{0}^{2}}{\gamma_{\x}},\gamma_\x/2)$, $\hat{\g}_{-1}\leftarrow\texttt{ComparisonGE}(\x-r_{0}\y,\frac{L_2 r_{0}^{2}}{\gamma_{\x}},\gamma_\x/2)$.\;
    
    Set $\g=\sqrt{1-\langle \hat{\g}_{-1},\hat{\g}_{0}\rangle^{2}}\hat{\g}_{1}-\sqrt{1-\langle \hat{\g}_{1},\hat{\g}_{0}\rangle^{2}}\hat{\g}_{-1}$.\;
    
    \Return $\widehat \u(\x,\y,\delta,\gamma_\x,\gamma_\y)=\g/\|\g\|.$ \label{lin:ComparisonHessVecReturn}
\end{algorithm2e}

\begin{theorem}\label{thm:triangle}
Let $f\colon\R^n\to\R$ with $L_1$-Lipschitz gradient and $L_2$-Lipschitz Hessian. For any $\gamma_\x,\gamma_\y>0$ and $\x,\y\in\R^d$ satisfying 
\begin{align*}
\|\nabla &f(\x)\|\geq\gamma_\x,  \quad\|\nabla^2f(\x)\|\ge\sqrt{L_2\epsilon},\\
&\|\y\|=1,  \qquad\quad|y_1|\geq\gamma_\y,
\end{align*}
where $y_1:=\<\y,\u_1\>$ with $\u_1$ being a eigenvector of $\nabla^2 f(\x)$ with the largest magnitude eigenvalue, \algo{Comparison-Triangle} outputs a vector $\widehat \u$ satisfying
\begin{align*}
\left\|\widehat \u-\frac{\nabla^2f(\x)\cdot \y}{\|\nabla^2f(\x)\cdot \y\|}\right\|\leq\hat{\delta}
\end{align*}
using $\tilde{O}\big(n\log\big(1/(\gamma_{\x}\gamma_{\y}^{2}\hat{\delta}^{2})\big)\big)$ queries.
\end{theorem}

Intuitively, \algo{Comparison-Triangle} use Taylor expansion $\nabla f(\x\pm r\y)\approx\nabla f(\x)\pm r \nabla^2f(\x)\cdot \y$ and solving the triangle in \fig{triangle}.

\begin{proof}
Denote $\x_+\coloneqq \x+r_0\y$ and $\x_-\coloneqq \x-r_0\y$. Since $f$ has $L_2$-Lipschitz Hessian,
\begin{align}\label{eq:HV-second-1}
\left\|\nabla f(\x_+)-\nabla f(\x)-r_{0} \nabla^2f(\x)\cdot\y\right\|\leq L_2 r_{0}^{2}/2,   \\
\left\|\nabla f(\x_-)-\nabla f(\x)+r_{0} \nabla^2f(\x)\cdot\y\right\|\leq L_2 r_{0}^{2}/2. \label{eq:HV-second-2}
\end{align}
Therefore, we have $\left\|\nabla f(\x_+)+\nabla f(\x_-)-2\nabla f(\x)\right\|\leq L_2 r_{0}^{2}$ and $\big\|\nabla^2f(\x)\cdot\y-\frac{\nabla f(\x_+)-\nabla f(\x_-)}{2r_0}\big\|\leq L_2 r_{0}/2$. Furthermore, as $r_0\leq \gamma_\x/(100L_1)$ and $f$ has $L_1$-Lipschitz gradient, both $\|\nabla f(\x_+)\|$ and $\|\nabla f(\x_-)\|$ is at least $\gamma_{\x}-L_1\gamma_\x/(100L_1)=0.99\gamma_{\x}$. We first understand how to approximate $\nabla^2f(\x)\cdot\y$ by normalized vectors $\frac{\nabla f(\x)}{\|\nabla f(\x)\|}, \frac{\nabla f(\x_+)}{\|\nabla f(\x_+)\|}, \frac{\nabla f(\x_-)}{\|\nabla f(\x_-)\|}$, and then analyze the approximation error due to using $\hat{\g}_{0}, \hat{\g}_{1}, \hat{\g}_{-1}$, respectively. By \lem{vector-norm-1}, we have
\begin{align*}
&\frac{1}{2\|\nabla f(\x)\|}\frac{\|\nabla f(\x)-r_{0}\nabla^{2}f(\x)\cdot\y\|}{\sqrt{1-\left\langle\frac{\nabla f(\x)+r_{0}\nabla^{2}f(\x)\cdot\y}{\|\nabla f(\x)+r_{0}\nabla^{2}f(\x)\cdot\y\|},\frac{\nabla f(\x)}{\|\nabla f(\x)\|}\right\rangle^{2}}}\nonumber\\
&\qquad=\frac{1}{2\|\nabla f(\x)\|}\frac{\|\nabla f(\x)+r_{0}\nabla^{2}f(\x)\cdot\y\|}{\sqrt{1-\left\langle\frac{\nabla f(\x)-r_{0}\nabla^{2}f(\x)\cdot\y}{\|\nabla f(\x)-r_{0}\nabla^{2}f(\x)\cdot\y\|},\frac{\nabla f(\x)}{\|\nabla f(\x)\|}\right\rangle^{2}}}.
\end{align*}
We denote the value above as $\alpha$. Because $f$ is $L_2$-Hessian Lipschitz, $\|r_{0} \nabla^2f(\x)\cdot\y\|\leq r_{0}L_2$. Since $r_{0}\leq\frac{\gamma_{\x}}{100L_2}$, $\|r_{0} \nabla^2f(\x)\cdot\y\|\leq\frac{\gamma_{\x}}{100}$. 
Also by \lem{dist-inner-vector} we have
\begin{align*}
&\left\langle\frac{\nabla f(\x)+r_{0}\nabla^{2}f(\x)\cdot\y}{\|\nabla f(\x)+r_{0}\nabla^{2}f(\x)\cdot\y\|},\frac{\nabla f(\x)}{\|\nabla f(\x)\|}\right\rangle\geq 0.94, \\
&\left\langle\frac{\nabla f(\x)-r_{0}\nabla^{2}f(\x)\cdot\y}{\|\nabla f(\x)-r_{0}\nabla^{2}f(\x)\cdot\y\|},\frac{\nabla f(\x)}{\|\nabla f(\x)\|}\right\rangle\geq 0.94,
\end{align*}
which promises that $\alpha\geq\frac{0.99}{2\sqrt{1-0.94^{2}}}\geq 1$. In arguments next, we say a vector $\u$ is $\xi$-close to a vector $\v$ if $\|\u-\v\|\leq \xi$ for any $\xi\geq0$. We prove that the vector
\begin{align}
\tilde{\g}_{1}&\coloneqq\frac{\nabla f(\x)}{\|\nabla f(\x)\|}+\alpha\cdot\Bigg(\sqrt{1-\left\langle\frac{\nabla f(\x_-)}{\|\nabla f(\x_-)\|},\frac{\nabla f(\x)}{\|\nabla f(\x)\|}\right\rangle^{2}}\frac{\nabla f(\x_+)}{\|\nabla f(\x_+)\|}\nonumber\\
&\qquad\qquad-\sqrt{1-\left\langle\frac{\nabla f(\x_+)}{\|\nabla f(\x_+)\|},\frac{\nabla f(\x)}{\|\nabla f(\x)\|}\right\rangle^{2}}\frac{\nabla f(\x_-)}{\|\nabla f(\x_-)\|}\Bigg)\label{eq:direction-vector+}
\end{align}
is $\frac{7L_2 r_{0}^{2}}{\gamma_{\x}}$-close to a vector proportional to $\nabla f(\x_+)$. This is because \eq{HV-second-1},  \eq{HV-second-2}, and \lem{dist-norm-vector} imply that $
\frac{\nabla f(\x_+)}{\|\nabla f(\x_+)\|}\text{\quad and\quad}\frac{\nabla f(\x)+r_{0}\nabla^{2}f(\x)\cdot\y}{\|\nabla f(\x)+r_{0}\nabla^{2}f(\x)\cdot\y\|}$ are $\frac{L_2 r_{0}^{2}}{0.99\gamma_{\x}}$-close to each other,
\begin{align}\label{eq:direction-vector+coeff}
\sqrt{1-\left\langle\frac{\nabla f(\x_-)}{\|\nabla f(\x_-)\|},\frac{\nabla f(\x)}{\|\nabla f(\x)\|}\right\rangle^{2}}\frac{\nabla f(\x_+)}{\|\nabla f(\x_+)\|}
\end{align}
is proportional to $\nabla f(\x_+)$, and the definition of $\alpha$ implies
\begin{align}
&\frac{\nabla f(\x)}{\|\nabla f(\x)\|}-\alpha\cdot\frac{\nabla f(\x)-r_{0}\nabla^{2}f(\x)\cdot\y}{\|\nabla f(\x)-r_{0}\nabla^{2}f(\x)\cdot\y\|}\cdot\sqrt{1-\left\langle\frac{\nabla f(\x)+r_{0}\nabla^{2}f(\x)\cdot\y}{\|\nabla f(\x)+r_{0}\nabla^{2}f(\x)\cdot\y\|},\frac{\nabla f(\x)}{\|\nabla f(\x)\|}\right\rangle^{2}}\nonumber\\
&\qquad=\nabla f(\x)+r_{0}\nabla^{2}f(\x)\cdot\y/(2\|\nabla f(\x)\|).\label{eq:HV-alpha-diff}
\end{align}
The above vector is $\frac{L_2 r_{0}^{2}}{4\gamma_{\x}}$-close to $\frac{\nabla f(\x_+)}{2\|\nabla f(\x)\|}$ by \eq{HV-second-1}, and the error in above steps cumulates by at most $\frac{6L_2 r_{0}^{2}}{0.99\gamma_{\x}}$ using \lem{dist-inner-vector}. In total, the error is at most $\frac{6L_2 r_{0}^{2}}{0.99\gamma_{\x}}+\frac{L_2 r_{0}^{2}}{4\gamma_{\x}}\leq \frac{7L_2 r_{0}^{2}}{\gamma_{\x}}$. Furthermore, this vector proportional to $\nabla f(\x_+)$ that is $\frac{L_2 r_{0}^{2}}{4\gamma_{\x}}$-close to \eq{direction-vector+} has norm at least $(1-0.01)/2=0.495$ because the coefficient in \eq{direction-vector+coeff} is positive, while in the equality above we have $\|r_{0} \nabla^2f(\x)\cdot\y\|\leq\frac{\gamma_{\x}}{100}$. Therefore, applying \lem{dist-norm-vector}, the vector $\tilde{\g}_{1}$ in \eq{direction-vector+} satisfies
\begin{align}\label{eq:HV-error-1}
\big\|\tilde{\g}_{1}/\|\tilde{\g}_{1}\|-\nabla f(\x_+)/\|\nabla f(\x_+)\|\big\|\leq 29L_2 r_{0}^{2}/\gamma_{\x}.
\end{align}
Following the same logic, we can prove that the vector
\begin{align}
\tilde{\g}_{-1}&\coloneqq\frac{\nabla f(\x)}{\|\nabla f(\x)\|}-\alpha\sqrt{1-\left\langle\frac{\nabla f(\x_-)}{\|\nabla f(\x_-)\|},\frac{\nabla f(\x)}{\|\nabla f(\x)\|}\right\rangle^{2}}\frac{\nabla f(\x_+)}{\|\nabla f(\x_+)\|}\nonumber\\
&\qquad +\alpha\sqrt{1-\left\langle\frac{\nabla f(\x_+)}{\|\nabla f(\x_+)\|},\frac{\nabla f(\x)}{\|\nabla f(\x)\|}\right\rangle^{2}}\frac{\nabla f(\x_-)}{\|\nabla f(\x_-)\|}\label{eq:direction-vector-}
\end{align}
satisfies
\begin{align}\label{eq:HV-error-2}
\big\|\tilde{\g}_{-1}/\|\tilde{\g}_{-1}\|-\nabla f(\x_-)/\|\nabla f(\x_-)\|\big\|\leq 29L_2 r_{0}^{2}/\gamma_{\x}.
\end{align}
Furthermore, \eq{HV-alpha-diff} implies that the vector $\tilde{\g}_{1}-\tilde{\g}_{-1}$ is $\frac{14L_2 r_{0}^{2}}{\gamma_{\x}}$-close to
\begin{align}
\frac{\nabla f(\x)+r_{0}\nabla^{2}f(\x)\cdot\y}{2\|\nabla f(\x)\|}-\frac{\nabla f(\x)-r_{0}\nabla^{2}f(\x)\cdot\y}{2\|\nabla f(\x)\|}
=r_{0}\nabla^{2}f(\x)\cdot\y/\|\nabla f(\x)\| \label{eq:HV-ideal}.
\end{align}
Given that $\|\nabla^2f(\x)\|\ge\sqrt{L_2\epsilon}$ and $|y_1|\geq\gamma_\y$, $\|\nabla^2f(\x)\cdot\y\|\geq \sqrt{L_2\epsilon}\gamma_\y$. Therefore, the RHS of \eq{HV-ideal} has norm at least $\frac{r_{0}\sqrt{L_2\epsilon}\gamma_\y}{\gamma_{\x}}$, and by \lem{dist-norm-vector} we have
\begin{align}
\left\|\frac{\tilde{\g}_{1}-\tilde{\g}_{-1}}{\|\tilde{\g}_{1}-\tilde{\g}_{-1}\|}-\frac{\nabla^2f(\x)\cdot\y}{\|\nabla^2f(\x)\cdot\y\|}\right\|\leq\frac{14L_2 r_{0}^{2}}{\gamma_{\x}}\Big/\Big(\frac{r_{0}\sqrt{L_2\epsilon}\gamma_\y}{\gamma_{\x}}\Big)=\frac{14r_{0}\sqrt{L_2}}{\sqrt{\epsilon}\gamma_{\y}}\label{eq:HV-error-3}.
\end{align}
Finally, by \thm{Comparison-GDE}, the error terms coming from $\texttt{ComparisonGE}$, $\big\|\hat{\g}_{0}-\frac{\nabla f(\x)}{\|\nabla f(\x)\|}\big\|$, $\ \big\|\hat{\g}_{1}-\frac{\nabla f(\x_+)}{\|\nabla f(\x_+)\|}\big\|$, and $\ \big\|\hat{\g}_{-1}-\frac{\nabla f(\x_-)}{\|\nabla f(\x_-)\|}\big\|$, are all upper bounded by $\frac{L_2 r_{0}^{2}}{\gamma_{\x}}$. Combined with \eq{HV-error-1} and \eq{HV-error-2}, we know that the vector $\g$ we obtained in \algo{Comparison-Triangle} is $\frac{61L_2 r_{0}^{2}}{\gamma_{\x}}$ close to $\frac{\tilde{\g}_{1}-\tilde{\g}_{-1}}{2\alpha}$. Since $\alpha\geq 1$, by \lem{dist-norm-vector} we have
\begin{align}\label{eq:HV-error-4}
\left\|\frac{\g}{\|\g\|}-\frac{\tilde{\g}_{1}-\tilde{\g}_{-1}}{\|\tilde{\g}_{1}-\tilde{\g}_{-1}\|}\right\|\leq\frac{61L_2 r_{0}^{2}}{\gamma_{\x}}.
\end{align}
Combining \eq{HV-error-3} and \eq{HV-error-4},
\begin{align}\label{eq:HV-error-5}
\left\|\frac{\g}{\|\g\|}-\frac{\nabla^2f(\x)\cdot\y}{\|\nabla^2f(\x)\cdot\y\|}\right\|\leq\frac{61L_2 r_{0}^{2}}{\gamma_{\x}}+\frac{14r_{0}\sqrt{L_2}}{\sqrt{\epsilon}\gamma_{\y}}.
\end{align}
Our choice of $r_{0}=\min\big\{\frac{\gamma_{\x}}{100L_1},\frac{\gamma_{\x}}{100L_2},\frac{\sqrt{\gamma_{\x}\hat{\delta}}}{20\sqrt{L_2}},\frac{\gamma_{\y}\hat{\delta}\sqrt{\epsilon}}{20\sqrt{L_2}}\big\}$ guarantees that the RHS of \eq{HV-error-5} is at most $\hat{\delta}$. In terms of query complexity, we make 3 calls to $\texttt{ComparisonGE}$. By \thm{Comparison-GDE}, the total query complexity is $O\big(n\log\big(nL_2 L_1^{2}/\gamma_{\x}\gamma_{\y}^{2}\epsilon\hat{\delta}^{2}\big)\big)$.
\end{proof}


\section{Robust Hessian Estimation by Comparisons}\label{sec:hess_est}
In this section, we show how to obtain an estimate of the normalized Hessian at a given point using comparisons. 

\begin{algorithm2e}[htbp!]
\caption{\texttt{ComparisonHE}$(\x, \delta)$: Estimate Normalized Hessian at $\x$ with Precision $\delta$ by Comparisons}
\label{algo:normhess}
\LinesNumbered
\DontPrintSemicolon
\KwIn{Target point $\x$, target accuracy $\delta>0$}
\KwOut{$\widehat H$}

Set $\eta \gets c_0\,\delta^4/n^2$, thresholds $\tau_\alpha,\tau_\beta \gets \Theta(\sqrt{\eta})$, $\sigma \gets \frac{1}{8}\eta^2$, sample $t\sim\mathrm{Unif}(\{1,\dots,n\})$.\;

Define $\y_i^{\pm}\gets \frac{\e_i\pm \sigma \e_t}{\sqrt{1+\sigma^2}}$ for all $i\in[n]$.\;

Define $\z_i^{\pm}\gets \frac{\e_1+\e_i\pm \sigma \e_t}{\sqrt{2+\sigma^2}}$ for all $i\in\{2,\dots,n\}$.\;

\For{$i=1,\dots,n$}{
  Query $\g_i^{0}\gets \widehat \u(\x,\e_i,\eta/4,\epsilon,\eta/2\sqrt{n})$, $\g_i^{+}\gets \widehat \u(\x,\y_i^{+},\eta/4,\epsilon,\eta/2\sqrt{n})$, $\g_i^{-}\gets \widehat \u(\x,\y_i^{-},\eta/4,\epsilon,\eta/2\sqrt{n})$.\;

  \eIf{$\langle \g_i^{+},\g_i^{-}\rangle \le 1-\tau_\alpha$}{
    Set $\g_i\gets \0$.\label{lin:line7}
  }{
    $(\v_1, \v_2)\gets \arg\max_{\v_1\neq \v_2\in\{\g_i^{0},\g_i^{+},\g_i^{-}\}}\langle \v_1, \v_2\rangle$.\;
    Set $\g_i\gets \frac{\v_1+\v_2}{\|\v_1+\v_2\|}$.\label{lin:line10}\;
  }
}

\For{$i=2,\dots,n$}{
  Query $\g_{1i}^{0}\gets \widehat \u(\x,\e_1+\e_i,\eta/4,\epsilon,\eta/2\sqrt{n})$, $\g_{1i}^{+}\gets \widehat \u(\x,\z_i^{+},\eta/4,\epsilon,\eta/2\sqrt{n})$, $\g_{1i}^{-}\gets \widehat \u(\x,\z_i^{-},\eta/4,\epsilon,\eta/2\sqrt{n})$.\;

  \eIf{$\langle \g_{1i}^{+},\g_{1i}^{-}\rangle \le 1-\tau_\beta$}{
    Set $\g_{1i}\gets \0$.\label{lin:line14}
  }{
    $(\v_1, \v_2)\gets \arg\max_{\v_1\neq \v_2\in\{\g_{1i}^{0},\g_{1i}^{+},\g_{1i}^{-}\}}\langle \v_1, \v_2\rangle$.\;
    Set $\g_{1i}\gets \frac{\v_1+\v_2}{\|\v_1+\v_2\|}$.\label{lin:line17}\;
  }
}

Set $\hat{r}_1\gets 1$.\;
\For{$i\gets 2$ \KwTo $n$}{
  $\widehat\alpha_i\gets \langle \g_i,\g_1\rangle$, $\widehat\beta_i\gets \langle \g_{1i},\g_1\rangle$.\;

  \eIf{$\widehat\beta_i \ge 1-\tau_\beta$}{\label{lin:beta}
    \eIf{$\widehat\alpha_i \le 1-\tau_\alpha$}{\label{lin:alpha}
      $\hat{r}_i\gets 0$.
    }{
      $\hat{r}_i\gets \texttt{PerturbAndSolve}(i,\eta)$.
    }
  }{\label{lin:non_degenerate}
    Solve the following equation with $(\alpha,\beta)=(\widehat\alpha_i,\widehat\beta_i)$ and take the unique root $\hat{r}_i\in[0,1]$.\;
    $(\beta^2-\alpha^2)\hat{r}_i^2 + 2\alpha(\beta^2-1)\hat{r}_i + (\beta^2-1)=0.$\;\label{lin:solver}
  }
}

$\widetilde H\gets \frac{1}{2}\left([\hat{r}_1 \g_1,\dots,\hat{r}_n \g_n]+[\hat{r}_1 \g_1,\dots,\hat{r}_n \g_n]^\top\right)$.\;\label{lin:forming}
\Return $\widehat H\gets \widetilde H/\|\widetilde H\|$.\;
\end{algorithm2e}

\begin{algorithm2e}[htbp!]
\caption{\texttt{PerturbAndSolve}$(i,\eta)$}
\label{algo:pert}
\DontPrintSemicolon
Set $\rho\gets \eta^{1/3}$.\;
\For{$j\gets 2$ \KwTo $n$}{
  \If{$j\neq i$ \KwSty{and} $\widehat\alpha_j\le 1-\tau_\alpha$ \KwSty{and} $\widehat\alpha_{1j}\le1-\rho^2$}{
    Query $\g_{i,j}\gets \widehat \u(\x,\e_i+\rho \e_j,\eta/4,\epsilon,\eta/2\sqrt{n})$.\;
    $\widehat\alpha_{i,j}\gets \langle \g_{i,j},\g_1\rangle$.\;
    Query $\g_{1,i,j}\gets \widehat \u(\x,\e_1+\e_i+\rho \e_j,\eta/4,\epsilon,\eta/2\sqrt{n})$.\;
    $\widehat\beta_{i,j}\gets \langle \g_{1,i,j},\g_1\rangle$.\;
    \If{$\langle \g_{i,j},\g_1\rangle\le 1-\rho$}{
      \Return $\hat{r}_i\gets 0$.\;
    }
    Solve the following equation with $(\alpha,\beta)=(\widehat\alpha_{i,j},\widehat\beta_{i,j})$ and take $\hat{r}_{i,j}\in[0,1+\rho]$.\;
    $(\beta^2-\alpha^2)\hat{r}_i^2 + 2\alpha(\beta^2-1)\hat{r}_i + (\beta^2-1)=0.$\;
    \Return $\hat{r}_i\gets \hat{r}_{i,j}$.\;
  }
}
\Return $\hat{r}_i\gets 1$.
\end{algorithm2e}

\begin{theorem}
\label{thm:main_app}
Let $\x$ satisfy $\|H\|\ge\sqrt{L_2\epsilon}$ and $\|\g\|\ge\epsilon$. There exist universal constants $c_0,K>0$ such that, \algo{normhess} makes $\tilde{O}(n^2\log\big(1/\delta\big))$ queries and outputs $\widehat H$ satisfying
\[
\left\|\widehat H-\frac{H}{\|H\|}\right\|\le K\sqrt{n}\,\eta^{1/4}.
\]
with probability at least $2/3$. 
\end{theorem}

We divide the analysis into a non-degenerate case described in \assum{nondeg_app} and two degenerate cases: either $\h_i$ has a sufficiently small norm, or the directions of $\h_1$ and $\h_i$ are too close. We handle the non-degenerate case in \sec{non-degenerate} and handle the two degenerate cases in \sec{degenerate} separately.

Denote $H=(\h_1,\ldots,\h_n)$ with $\h_i=H\e_i$ and define $r_i:=\|\h_i\|/\|\h_1\|\in[0,1]$. Without loss of generality, we assume$\|H\e_1\|=\max_i\|H\e_i\|$. Otherwise, by testing $\inner{\frac{\h_i}{\|\h_i\|}}{\frac{\h_i+\h_j}{\|\h_i+\h_j\|}}$ and $\inner{\frac{\h_j}{\|\h_j\|}}{\frac{\h_i+\h_j}{\|\h_i+\h_j\|}}$, we can roughly tell which one is larger and then relabel columns. The error only incurs a constant overhead in the complexity.  

Our estimator first estimates the \emph{directions} $\h_i/\|\h_i\|$ and $(\h_1+\h_i)/\|\h_1+\h_i\|$ via Hessian-vector product queries at (perturbed) $\e_i$ and $\e_1+\e_i$, and then computes $\hat{r}_i$ by applying \lem{vector-norm-1} in parallelogram spanned by $\h_1$, $\h_i$ and $\h_1+\h_i$ in non-degenerate cases that the direction of $\h_1$ and $\h_i$ are not close.

The computation of $\hat{r}_i$ can be roughly stated as follows: for each $i\ge 2$ with $\h_i\neq \0$ and $\h_1+\h_i\neq 0$, define unit vectors
\[
\u_1:=\frac{\h_1}{\|\h_1\|},\quad
\u_i:=\frac{\h_i}{\|\h_i\|},\quad
\u_{1i}:=\frac{\h_1+\h_i}{\|\h_1+\h_i\|}.
\]
Let $\alpha_i:=\langle \u_i,\u_1\rangle$ and $\beta_i:=\langle \u_{1i},\u_1\rangle$.
A direct computation yields
\begin{equation}\label{eq:F_main}
\beta_i = F(r_i,\alpha_i)
:=\frac{1+r_i\alpha_i}{\sqrt{1+r_i^2+2r_i\alpha_i}}.
\end{equation}
Squaring \eq{F_main} gives the quadratic equation solved in \lin{solver} in \algo{normhess}
\begin{equation}\label{eq:quad}
(\beta_i^2-\alpha_i^2)\hat{r}_i^2 + 2\alpha_i(\beta_i^2-1)\hat{r}_i + (\beta_i^2-1)=0.
\end{equation}

By solving this equation, we can roughly estimate $r_i$, thus we can form the target $\widehat H$ (see \lin{forming}). The analysis has two basic cases: \emph{(i)} When $\h_1$ and $\h_i$ overlap relatively small in directions and their norms are not too small, i.e. $\alpha_i$ and $\beta_i$ are bounded away from 1, we can solve \eq{quad} with bounded error on $|\hat{r}_i-r_i|$. Otherwise, when $\alpha_i$ and $\beta_i$ are close to 1, \eq{quad} becomes degenerate. There are two cases that may cause degeneracy:  \emph{(ii)} When $\|\h_i\|$ is relatively small, we can tell this case by testing if the angle between $\u_1$ and $\u_i$ is relatively large, and we can directly treat $\|\h_i\|=0$ since the error of column vectors cause bounded error to Hessian estimation. \emph{(iii)} When $\h_1$ and $\h_i$ have close directions, to avoid uncontrollable error, we apply \algo{pert} to enforce an angle separation $1-\alpha^2=\Omega(\tau_\alpha)$, in which we use perturbed $\e_i+\rho \e_j$ to replace $\e_i$, and perturbed $\e_1+\e_i+\rho \e_j$ to replace $\e_1+\e_i$ to guarantee the non-degeneracy in the computation of the ratio $\hat{r}_i$. This makes the inverse map Lipschitz with coefficient $O(1/\rho^2)$, giving a noise term $O(\eta/\sqrt{\tau_\alpha}\rho^2)$. The perturbation changes the effective triangle by $O(\rho/\sqrt{\tau_\alpha})$, producing a bias term $O(\rho/\sqrt{\tau_\alpha})$. Hence
\[
|\hat{r}_i-r_i|\le \frac{1}{\sqrt{\tau_\alpha}}(C\frac{\eta}{\rho^2}+B\rho),
\]
optimizing it gives $|\hat{r}_i-r_i|=O(\eta^{1/4})$. (See \lem{tradeoff_app}.) 

With constant probability (guaranteed by \lem{coordinate-bounds}), our sampled $t$ satisfies $|\langle \e_t, \v\rangle|\ge1/\sqrt{n}$, where $\v$ is the eigenvector corresponding to minimum eigenvalue of $H$, here we use the isotropy of $n$ coordinates. Once we can satisfy a $1/\sqrt{n}$-overlap, our choice of $\g_i$ and $\g_{1i}$ satisfies $\|\g_i-\u_i\|\le \eta$ and $\|\g_{1i}-\u_{1i}\|\le \eta$ according to \lem{suit_tri}. Inner products satisfy $|\widehat\alpha_i-\alpha_i|\le O(\eta)$ and $|\widehat\beta_i-\beta_i|\le O(\eta)$.

When $\alpha_i$ is bounded away from $1$ and $r_i$ is not tiny, the implicit map $\beta=F(r,\alpha)$ has slope $|\partial F/\partial r|=\Omega(1/r)$, so solving the quadratic \eq{quad} gives $|\hat{r}_i-r_i|=O(\eta/r_{\min}\tau_\alpha)$. (See \lem{stable_nondeg_app}.)

Finally, assembling columns $\hat{r}_i g_i$ yields a scale-free matrix $\widetilde H$ satisfying $\|\widetilde H-H^\star\|\le O(\sqrt{n}\eta^{1/4})$, and normalization stability implies the same order bound for $\widehat H-\frac{H}{\|H\|}$. (See \lem{normstab_app} and \lem{col2spec_app}.)

\subsection{The non-degenerate case}\label{sec:non-degenerate}
Here we give a deterministic bound for the case that directly solves \eq{quad} without perturbation.
\begin{assumption}[Non-degenerate case]\label{assum:nondeg_app}
Fix an index $i\ge 2$ such that $\h_i\neq \0$ and $\h_1+\h_i\neq \0$.
Assume:
\begin{enumerate}
\item (Not too parallel) $\alpha_i \le 1-\gamma$ for some $\gamma\in(0,1)$.
\item (Not too small) $r_i \ge r_{\min}$ for some $r_{\min}\in(0,1]$.
\end{enumerate}
\end{assumption}

\begin{lemma}\label{lem:stable_nondeg_app}
Under \assum{nondeg_app}, let $(\hat\alpha,\hat\beta)$ satisfy
$|\hat\alpha-\alpha_i|\le \varepsilon$ and $|\hat\beta-\beta_i|\le \varepsilon$,
with $\varepsilon\le \gamma/10$. Let $\hat{r}$ be the (unique) solution in $[0,1]$ of $\hat\beta=F(r,\hat\alpha)$ introduced in \eq{F_main}. Then
\[
|\hat{r}-r_i|
\le
\frac{C_0}{r_{\min}\gamma}\,\varepsilon,
\]
for some universal constant $C_0$.
\end{lemma}

\begin{proof}
By the mean value theorem,
\[
|F(\hat{r},\hat{\alpha})-F(r_i,\hat\alpha)|
=
|\frac{\partial F}{\partial r}(\xi,\hat\alpha)|\,|\hat{r}-r_i|
\]
for some $\xi$ between $\hat{r}$ and $r_i$. Therefore,
\[
|\hat{r}-r_i|
=
\frac{|\hat\beta-F(r_i,\hat\alpha)|}{|\frac{\partial F}{\partial r}(\xi,\hat\alpha)|}.
\]
We bound numerator and denominator respectively. By triangle inequality,
\begin{equation}\label{eq:tri_bound}
|\hat\beta-F(r_i,\hat\alpha)|
\le
|\hat\beta-\beta_i|
+
|F(r_i,\alpha_i)-F(r_i,\hat\alpha)|
\le
\epsilon + L_\alpha \epsilon,
\end{equation}
where $L_\alpha$ is a Lipschitz constant of $F$ in $\alpha$ on the compact set
$r\in[0,1]$, $\alpha\in[-1,1]$. We can bound $L_\alpha$ explicitly by computing
$\partial F/\partial \alpha$ and taking a supremum; a crude universal bound
$L_\alpha\le 2$ suffices (indeed, one can check $|\partial F/\partial\alpha|\le 2$
on this domain). Thus the numerator is less than $3\epsilon$.

On the other hand, We have
\[
\left|\frac{\partial F}{\partial r}(\xi,\hat\alpha)\right|
=
\frac{\xi(1-\hat\alpha^2)}{(1+\xi^2+2\xi\hat\alpha)^{3/2}}.
\]
Since $\xi$ lies between $\hat{r}\in[0,1]$ and $r_i\ge r_{\min}$, we have
$\xi\ge r_{\min}/2$ provided $\epsilon$ is small enough; we can ensure this by restricting to $\epsilon\le \gamma/10$ and using continuity of the inverse map. Also, $\hat\alpha \le \alpha_i+\epsilon \le 1-\gamma+\frac\gamma{10} = 1-\tfrac{9}{10}\gamma$,
so
\[
1-\hat\alpha^2 = (1-\hat\alpha)(1+\hat\alpha)
\ge \left(\tfrac{9}{10}\gamma\right)\cdot 1 = \tfrac{9}{10}\gamma.
\]
Therefore
\[
\left|\frac{\partial F}{\partial r}(\xi,\hat\alpha)\right|
\ge
\frac{(r_{\min}/2)\cdot (\tfrac{9}{10}\gamma)}{8}
\ge
c\,r_{\min}\gamma
\]
for a universal constant $c>0$. Combining numerator and denominator gives
\[
|\hat{r}-r_i|
\le
\frac{3\varepsilon}{c\,r_{\min}\gamma}
=
\frac{C_0}{r_{\min}\gamma}\,\varepsilon
\]
with $C_0:=3/c$.
\end{proof}

\begin{corollary}\label{cor:nondeg_eta_app}
Under \assum{nondeg_app}, the variables $\widehat\alpha_i,\widehat\beta_i$ in \eq{quad} of \algo{normhess} satisfies
$|\widehat\alpha_i-\alpha_i|\le 3\eta$ and $|\widehat\beta_i-\beta_i|\le 3\eta$,
and thus
\[
|\hat{r}_i-r_i|\le \frac{C_1}{r_{\min}\tau_\alpha}\eta
\]
for a constant $C_1$.
\end{corollary}

\subsection{The degenerate cases}\label{sec:degenerate}
In this subsection, we discuss about how to deal with the two degenerate cases that lead to the failure of \lem{stable_nondeg_app}.
In the case that $\h_i$ has a sufficiently small norm, we introduce the following lemma.
\begin{lemma}\label{lem:norm_small}
If $\hat{r}_i$ is set to $0$, which means the algorithm
enters this case only when the column is treated as negligible, i.e. $r_i\le r_{\min}$. We have
\[
r_i\le r_{\min}.
\]
\end{lemma}
\begin{proof}
If $\hat{r}_i$ is set to $0$, by \lin{beta} and \lin{alpha} in \algo{normhess}, we have 
\begin{equation*}
\widehat{\alpha}_i\le 1-\tau_\alpha.
\end{equation*}

Geometry tells us 
\begin{equation}\label{eq:r_min}
\tau_\alpha=\cos(\arctan r_{\min})
\end{equation}

Plugging in $\widehat{\alpha}_i=\inner{\g_i}{\g_1}$ and \eq{r_min} we get
\[
r_i\le r_{\min}.
\]
\end{proof}

In the case that the directions of $\h_1$ and $\h_i$ are too close, we can find an index $j$ such that the direction of  $\h_j$ and the direction of $\h_1$ (so to $\h_i$) are sufficiently separated for the perturbation step. Note that if we cannot find the index $j$, it means that the rank of $H$ is $1$ up to a small error, which would be dealt with by line search step in \lem{cand2_descent}, so we have the following assumption.

\begin{assumption}\label{assum:sep_app}
Fix an index $i$ such that $\h_i\neq \0$ and $\h_1+\h_i\neq \0$. Assume $r_i\ge r_{\min}>0$.
\algo{pert} finds an index $j$ such that
\begin{equation}\label{eq:sep_app}
1-\inner{\frac{\h_i+\rho \h_j}{\|\h_i+\rho \h_j\|}}{\frac{\h_1}{\|\h_1\|}}^2 \ge c_0\rho^2
\end{equation}
for some absolute $c_0>0$.
\end{assumption}

Define $\alpha := \inner{\v}{\u_1}, \w := \frac{\u_1+r_i \v}{\|\u_1+r_i \v\|},\beta := \inner{\w}{\u_1}$, then $\beta = F(r_i,\alpha)$. \algo{pert} estimates $\v$ and $\w$ by oracle queries
\[
\hat \v := \widehat \u(\e_i+\rho \e_j;\eta),\qquad
\hat \w := \widehat \u(\e_1+\e_i+\rho \e_j;\eta),
\]
and computes
\[
\hat\alpha := \inner{\hat \v}{\g_1},\qquad \hat\beta := \inner{\hat \w}{\g_1},
\]
then returns $\hat{r}$ as the solution in $[0,1+\rho]$ to
$\hat\beta=F(\hat{r},\hat\alpha)$.

\begin{lemma}\label{lem:amp_app}
Under \assum{sep_app}, there exists a constant $C$ such that
for $\eta$ sufficiently small,
\[
|\hat{r}-r_i| \le C\,\frac{\eta}{r_{\min}\rho^2}.
\]
\end{lemma}

\begin{proof}
By \thm{triangle}, $\|\hat \v-\v\|\le \eta$, $\|\hat \w-\w\|\le \eta$, and also $\|\g_1-\u_1\|\le \eta$. Apply \lem{dist-inner-vector} we have
\begin{equation}\label{eq:ab_err_app}
|\hat\alpha-\alpha|\le 3\eta,\qquad |\hat\beta-\beta|\le 3\eta.
\end{equation}
and
\[
\left|\frac{\partial F}{\partial r}(r_i,\alpha)\right|
=
\frac{r_i(1-\alpha^2)}{(1+r_i^2+2r_i\alpha)^{3/2}}
\ge
\frac{r_{\min}(1-\alpha^2)}{8}.
\]
Using the separation $1-\alpha^2\ge c_0\rho^2$, we obtain
\begin{equation}\label{eq:dF_lb_app}
\left|\frac{\partial F}{\partial r}(r_i,\alpha)\right|\ge \frac{r_{\min}c_0}{8}\rho^2
\end{equation}
By the mean value theorem,
\[
|\hat{r}-r_i|
=
\frac{|\hat\beta-F(r_i,\hat\alpha)|}{|\frac{\partial F}{\partial r}(\xi)|}
\]
for some $\xi$ between $r_i$ and $\hat{r}$. For $\eta$ sufficiently small, $\hat\alpha$ is close to $\alpha$ and thus $1-\hat\alpha^2 \ge \frac12(1-\alpha^2)\ge \frac12 c_0\rho^2$, and $\xi$ remains in a compact interval containing $[r_{\min}/2,1+\rho]$. Therefore, the lower bound \eq{dF_lb_app} holds for $|\frac{\partial F}{\partial r}(\xi)|$ as well
\[
\left|\frac{\partial F}{\partial r}(\xi,\hat\alpha)\right|
\ge \frac{r_{\min}c_0}{16} \rho^2.
\]
Hence
\[
|\hat{r}-r_i|
\le
\frac{16}{r_{\min}c_0\rho^2}\,|\hat\beta-F(r_i,\hat\alpha)|.
\]
Similar to \eq{tri_bound}, we have
\[
|\hat\beta-F(r_i,\hat\alpha)| \le |\hat\beta-\beta|+|F(r_i,\alpha)-F(r_i,\hat\alpha)| \le 3\eta+3L_\alpha\eta.
\]
where $L_\alpha$ is a Lipschitz constant of $F$ in $\alpha$ on the compact set
$r\in[0,1]$, $\alpha\in[-1,1]$. Conclude that the numerator is $\le C'\eta$, and therefore
\[
|\hat{r}-r_i|\le C\frac{\eta}{r_{\min}\rho^2}.
\]
\end{proof}

\begin{lemma}\label{lem:tradeoff_app}
Under \assum{sep_app} and assuming $r_i\ge r_{\min}$, the perturbation case estimator satisfies
\[
|\hat{r}-r_i| \le \frac{1}{r_{\min}}\left(C\,\frac{\eta}{\rho^2} + B\rho\right)
\]
for constants $C,B$ depending only on $c_0$.
\end{lemma}

\begin{proof}
The first term $C\eta/\rho^2$ follows from \lem{amp_app}.
The second term accounts for the fact that the perturbation uses $\v$ (direction
of $\h_i+\rho \h_j$) rather than $\u_i$ (direction of $\h_i$). By \lem{bias_app},
$\|\v-\u_i\|\le O(\rho)$ (with explicit constant $4/r_{\min}$).
On the non-degenerate set enforced by $1-\alpha^2\ge c_0\rho^2$, the mapping from
input directions to the recovered ratio is Lipschitz; therefore this $O(\rho)$
direction change induces an $O(\rho)$ bias in $r$. Combining the two terms of error yields the stated bound.
\end{proof}

\subsection{Proof of \thm{main_app}}
Here complete the proof of \thm{main_app}, the main theorem of this section.
\begin{proof}
By \lem{coordinate-bounds}, we have $\Pr_{\u \sim \mathrm{Unif}(S_1(0))}[|\langle \u, \v\rangle| \ge 1/\sqrt{n}]\ge 2/5$ where $\v$ is the eigenvalue of $H$ with the largest absolute eigenvalue. Because of the isotropy of n coordinates, $\v\sim \mathrm{Unif}(S_1(0))$. So we have
\begin{align*}
\Pr_{\u \sim \mathrm{Unif}(S_1(0))}[|\langle \u, \v\rangle| \ge 1/\sqrt{n}]&=\Pr_{\u \sim \mathrm{Unif}(S_1(0)), \v \sim \mathrm{Unif}(S_1(0))}[|\langle \u, \v\rangle| \ge 1/\sqrt{n}]\\
&=\Pr_{\u \sim \mathrm{Unif}(\{\e_1,\ldots,\e_n\}), \v \sim \mathrm{Unif}(S_1(0))}[|\langle \u, \v\rangle| \ge 1/\sqrt{n}]\\
&=\Pr_{t \sim \mathrm{Unif}(\{1,\ldots,n\}), \v \sim \mathrm{Unif}(S_1(0))}[|\langle \e_t, \v\rangle| \ge 1/\sqrt{n}]\\
&\ge 2/5.
\end{align*}
Repeat for $O(1)$ times and the success probability can be more than $2/3$. Let $H^\star$ be the scale-free target matrix \eq{Hstar_app}, so that $H/\|H\|=H^\star/\|H^\star\|$. Let $\widetilde H=[\hat{r}_1 \g_1,\dots,\hat{r}_n \g_n]$ and define $E:=\widetilde H-H^\star$. Fix $i$. If $\hat{r}_i$ is produced by the non-degenerate solver, \cor{nondeg_eta_app} gives
\[
|\hat{r}_i-r_i|\le \frac{C}{r_{\min}\tau_\alpha}\eta
\]
for some constant $C$. 

In the case that $\h_i$ has a sufficiently small norm, by \lem{norm_small} we have 
\[
|\hat{r}_i-r_i|<r_{\min}.
\]

In the case that the directions of $\h_1$ and $\h_i$ are too close, \eq{r_min} tells us 
\[
r_{\min}=\tan{(\arccos(1-\tau_\beta))}.
\]
When $\eta$ is sufficiently small, through Taylor expansion, $f(x)=\tan(\arccos(1-x))\sim \sqrt{2x}+O(x^{3/2})$, i.e. $r_{\min}\sim \sqrt{\tau_\beta}$.
If $\hat{r}_i$ is produced by perturbation, then \lem{tradeoff_app} yields
\[
|\hat{r}_i-r_i|\le \frac{1}{r_{\min}}\left(C\,\frac{\eta}{\rho^2} + B\rho\right).
\]

Combine these three cases and set $\rho=\Theta(\eta^{1/3})$, $\tau_\alpha=\tau_\beta=\Theta(\sqrt{\eta})$, we get
\[
|\hat{r}_i-r_i|\le O(\eta^{1/4}),
\]
and
\[
\|E\e_i\|
=
\|\hat{r}_i \g_i-r_i \u_i\|
\le
\|(\hat{r}_i-r_i)\g_i\| + \|r_i(\g_i-\u_i)\|
\le
|\hat{r}_i-r_i|\cdot\|\g_i\| + r_i\eta
\le
2|\hat{r}_i-r_i|+\eta,
\]
since $\|\g_i\|\le 1+\eta\le 2$ and $r_i\le 1$. Thus $\|E\e_i\|\le K_1\eta^{1/4}$ for a constant $K_1$. By \lem{col2spec_app},
\[
\|E\|\le K_1\sqrt{n}\,\eta^{1/4}.
\]
By \lem{Hstar_lb_app}, $\|H^\star\|\ge 1$. If $\eta$ is small enough that $\|E\|\le 1/2$, then \lem{normstab_app} gives
\[
\left\|\frac{\widetilde H}{\|\widetilde H\|}-\frac{H^\star}{\|H^\star\|}\right\|
\le 4\|E\|
\le 4K_1\sqrt{n}\,\eta^{1/4}.
\]
Using $H/\|H\|=H^\star/\|H^\star\|$ then gives the stated bound with $K:=4K_1$.
\end{proof}

\begin{corollary}
Under the setting of \thm{main_app}, we can increase success probability to at least $1-p$ using $O(n^2\log(1/\delta)\log(1/p))$ queries for any $p\in(0,1)$ to get the same result.
\end{corollary}

\begin{proof}
Consider the $m=\lceil18\ln(1/p)\rceil$ independent base runs in the algorithm. Define $\widehat H^{(r)}$ to be the output of the $r$-th iteration and set $\mathcal N_r=\{s\in\{1,\ldots,m\}:\|\widehat H^{(r)}-\widehat H^{(s)}\|\le 2\delta/3\}$. We finally return any $\widehat H^{(r)}$ with $|\mathcal N_r|\ge m/2$ and prove that it satisfies our requirements with success probability at least $1-p$. Let $X_r$ be the indicator that the $r$-th candidate satisfies $\|\widehat H^{(r)}-H\|\le \delta_0$. Since $\Pr[X_r=1]\ge2/3$, Hoeffding's inequality gives
\[
\Pr\left[\sum_{r=1}^m X_r\le m/2\right]\le \exp(-m/18)\le p.
\]
Conditioned on the complementary event, more than half of the candidates are good. Any two good candidates are within $2\delta_0=2\delta/3$ of each other, so every good candidate satisfies the selection condition $|\mathcal N_r|\ge m/2$. Thus the algorithm can select at least one candidate. Moreover, for any selected candidate $\widehat H^{(r)}$, the set $\mathcal N_r$ intersects the set of good candidates because both sets have size greater than $m/2$ or at least $m/2$ with a strict majority of good candidates. Therefore, for some good $\widehat H^{(s)}$,
\[
\|\widehat H^{(r)}-H\|\le \|\widehat H^{(r)}-\widehat H^{(s)}\|+\|\widehat H^{(s)}-H\|\le 2\delta/3+\delta/3=\delta.
\]
The total query complexity is $m\cdot O(n^2\log\frac{1}{\delta})=O(n^2\log\frac{1}{\delta}\log\frac{1}{p})$.
\end{proof} 


\section{Estimate the Ratio of the Hessian and the Gradient Norm}
\label{sec:grad_hess_ratio}
In this section, we show that, under the following assumption, the ratio $\frac{\|\nabla^2 f(\mathbf{x})\|}{\|\nabla f(\mathbf{x})\|}$ can be estimated using a single call to \texttt{ComparisonHE} (\algo{normhess}) and two calls to $\texttt{ComparisonGE}$ (\thm{Comparison-GDE}).

\begin{assumption}
\label{assum:hess_lip}
\begin{enumerate}
    \item The Hessian $H$ is neither positive definite nor negative definite.
    \item The Hessian $H$ is not rank-one.
    \item There exists a unit eigenvector $\u$ of $H$ with
eigenvalue $\lambda_\u$ such that
\begin{equation}\label{eq:angle_beta}
\arccos\inner{\frac{\g}{\|\g\|}}{\u}\ge \beta
\quad\text{and}\quad
\|H\u\|=|\lambda_\u|\ge \lambda,
\end{equation}
for some $\beta\in(0,\pi/2)$ and $\lambda>0$.
\end{enumerate}
\end{assumption}

\begin{algorithm2e}[htbp!]
\caption{\texttt{ComparisonRatio}$(\x,\delta_\kappa)$: Estimate Gradient-Hessian Ratio at point $\x$ with precision $\delta_\kappa$ by Comparisons}
\label{algo:Comparison-ratio}
\LinesNumbered
\DontPrintSemicolon
\KwIn{Point $\x$, lower bound of gradient $\epsilon$ precision $\delta_\kappa$, degenerate parameters $\beta$, $\lambda$}
\KwOut{Estimate $\hat\kappa \approx \|H\|/\|\g\|$}.

$\mu\gets\tan^2\beta\cdot\epsilon^2\delta_\kappa^2$.\;
$(\xi_1,\xi_2,\xi_3)\gets(O(\epsilon\delta_\kappa), O(\epsilon\delta_\kappa), O(\epsilon^2\delta_\kappa^2/n))$.\;

$\widehat H \gets \texttt{ComparisonHE}(\x;\xi_3)$.\label{lin:HE_in_ratio}\;
Compute a unit eigenvector $v$ of $\widehat H$ associated with the eigenvalue
$\lambda_{v}$ with the largest absolute value\;
$\widehat \g_1 \gets \texttt{ComparisonGE}(\x;\xi_1)$.\label{lin:GE_in_ratio}\;
$\widehat \g_2 \gets \texttt{ComparisonGE}(\x+\mu \v;\xi_2)$.\;

$\hat\kappa \gets \dfrac{\sqrt{1-\inner{\widehat \g_1}{\widehat \g_2}^2}}
{\mu\,\sqrt{1-\inner{\widehat \g_2}{\v}^2}}$.\;
\Return{$\hat\kappa$}.\;
\end{algorithm2e}

\begin{theorem}
\label{thm:gradient-Hessian-ratio}
Let $f\colon\mathbb R^n\to\mathbb R$ with $L_1$-Lipschitz gradient and $L_2$-Lipschitz Hessian. Assume that \assum{hess_lip} holds. Given any precision $\delta_\kappa>0$, there exists an algorithm $\texttt{ComparisonRatio}(\x,\delta_\kappa)$ such that, for any $\x\in\R^n$ with $\|\g\|\ge \epsilon$, it outputs $\hat\kappa$ satisfying
\begin{equation}\label{eq:ratio_bound_main}
\left|\hat\kappa-\frac{\|H\|}{\|\g\|}\right| \le \delta_\kappa.
\end{equation}
using $\tilde{O}(n^2\log\big(1/\delta_\kappa\big))$ queries. 
\end{theorem}

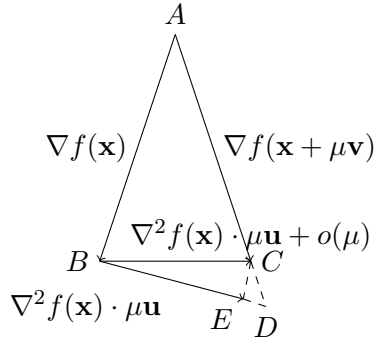
\begin{figure}[h]
\begin{tikzpicture}

\coordinate (Top) at (0,3);
\node[above] at (Top) {$A$};
\coordinate (L) at (-1,0);
\node[left] at (L) {$B$};
\coordinate (R) at (1,0);
\node[right] at (R) {$C$};

\coordinate (D) at (0.9,-0.5);
\node[below left] at (D) {$E$};

\coordinate (I) at (1.2,-0.6);
\node[below] at (I) {$D$};

\draw[->] (Top) -- (L)
    node[midway,left] {$\nabla f(\x)$};
\draw[->] (Top) -- (R)
    node[midway,right] {$\nabla f(\x+\mu \v)$};

\draw[->] (L) -- (R)
    node[above] {$\nabla^2f(\x)\cdot\mu \u+o(\mu)$};

\draw[->] (L) -- (D)
    node[midway,below left] {$\nabla^2f(\x)\cdot\mu \u$};

\draw[dashed] (D) -- (I);
\draw[dashed] (R) -- (I);

\draw[dashed] (R) -- (D);

\end{tikzpicture}
\centering
\caption{The intuition for computing the ratio.}
\label{fig:triangle}
\end{figure}

By \lem{taylor_grad}, we have the second-order expansion for some small $\mu$,
\begin{equation}\label{eq:taylor_grad_main}
\nabla f(\x+\mu \v)=\g+\mu H\v+r_\mu, \qquad \|r_\mu\|\le \tfrac{1}{2}L_2\mu^2.
\end{equation}

Here we choose $\v$ as an eigenvector of $\widehat H$ corresponding to the eigenvalue $\lambda_\v$, which is the largest magnitude eigenvalue satisfying $\arccos\inner{\frac{\g_1}{\|\g_1\|}}{\v}\ge \beta$ and $\|\widehat H\v\|=|\lambda_\v|\ge \lambda$. Otherwise, when finding the stationary point, we apply line search on direction $\v$ to handle this case. As shown in \fig{triangle}, we can use line segment BE to replace BC. By law of sines, we can get an estimation of the Hessian-gradient norm ratio
\[
\frac{\|\nabla^2f(\x)\cdot\mu\u\|}{\|\nabla f(\x)\|}=\frac{\mu\lambda_\u}{\|\nabla f(\x)\|}\approx\mu\frac{\|\nabla^2f(\x)\|}{\|\nabla f(\x)\|}.
\]
where $\u$ is an eigenvector of the real Hessian $H$ appropriately chosen.

A core lemma we use here is the Davis-Kahan Theorem (\lem{Davis-Kahan}), which bounds the distance of two subspaces spanned by eigenvectors under perturbation. We use this tool to bound $\|\u-\v\|$ and thus get an upper bound of the error $|\hat \kappa-\frac{\|H\|}{\|\g\|}|$. We will rigorously describe how to choose $\u$ in \sec{choose_u}, here we directly claim that $\u$ has two properties: 
\begin{itemize}
\item[\emph{(i)}] $\|\u-\v\|\le O\left(\frac{n\xi_3}{\epsilon\delta_\kappa}\right)$;
\item[\emph{(ii)}] $\|H\|-\|H\u\|\le\frac{\epsilon\delta_\kappa}{2}+2\xi_3$.
\end{itemize}
We can take $\xi_3$ sufficiently small to apply \lem{Davis-Kahan}, which gives $\|\sin(\Theta(U,V))\|\le\frac{\xi_3}{\lambda_k'-\lambda_{k+1}'}$. With these two properties we can get an upper bound of $|\hat\kappa-\frac{\|H\|}{\|\g\|}|$.

\subsection{Analysis of errors when estimating the Hessian-gradient norm ratio}\label{sec:choose_u}
Here we give some lemmas to prove the properties of $\u$, and first we describe how to choose $\u$.

Denote $\v$ as the eigenvector with the largest magnitude eigenvalue of $\widehat H$. Assume that all n eigenvalues of $\widehat H$ are $\lambda_1'\ge\cdots\ge\lambda_n'$ (correspond to eigenvectors $\v_1,\ldots,\v_n$), and due to \assum{hess_lip}, we can assume $\lambda_1'>0$ and $\lambda_n'<0$. Without loss of generality, we suppose that $\lambda_\v=\lambda_1'$ here, another case is the same. By the Dirichlet Theorem, there exists $\lambda_k'-\lambda_{k+1}'\ge \frac{\|\widehat H\|}{n-1}$ for some index $k$. We denote $k$ as the first index that $\lambda_k'-\lambda_{k+1}'\ge \epsilon\delta_\kappa/2n$. We call the set $\widehat \Lambda=\{\lambda_1',\ldots,\lambda_k'\}$ a cluster, and denote $V=\text{span}(\v_1,\ldots,\v_k)$. Denote the eigenvalues of Hessian $H$ as $\lambda_1\ge\cdots\ge\lambda_n$ (correspond to eigenvectors $\u_1,\ldots,\u_n$). By Weyl's Theorem (\lem{Weyl}) we know that $|\lambda_i-\lambda_i'|\le\xi_3$. Similarly define the cluster $\Lambda=\{\lambda_1,\ldots,\lambda_{k'}\}$ and space $U=\text{span}(u_1,\ldots,u_{k'})$. Denote the projection of $\v$ in $U$ as $\u$. 
\begin{lemma}\label{lem:choose_eigenvector}
 With $\u$ being chosen as above, we guarantee
\begin{equation*}
\|H\v\|-\|H\u\|\le \frac{\epsilon\delta_\kappa}{2}+\xi_3.
\end{equation*}
As a result, we have
\begin{equation}\label{eq:estimate_norm}
\|H\|-\|H\u\|\le \frac{\epsilon\delta_\kappa}{2}+2\xi_3.
\end{equation}
\end{lemma}

\begin{figure}[htbp!]
\centering
\caption{The intuition for applying the Davis-Kahan Theorem (\lem{Davis-Kahan}) to prove $\|H\|-\|H\u\|\le\frac{\epsilon\delta_\kappa}{2}+2\xi_3$.}
\begin{tikzpicture}

\draw[thick] (1,0) -- (5,0) node[right] {eigenvalues of $\widehat H$, as a sequence $\lambda_1\ge\cdots\ge\lambda_n$};
\draw[thick] (1,2) -- (5,2) node[right] {eigenvalues of $H$, as a sequence $\lambda_1'\ge\cdots\ge\lambda_n'$};

\foreach \x in {1,2,4} \fill (\x,0) circle (2pt);
\foreach \x in {1,1.5,2,2.3,3.7,4} \fill (\x,2) circle (2pt);

\node at (1, 0) [below] {$\lambda_{\v}(\| \widehat{H} \|)$};
\node at (2, 0) [below] {$\lambda_k'$};
\node at (4, 0) [below] {$\lambda_{k+1}'$};

\draw[dashed] (1, 0) -- (1.5, 2) node[midway,left] {Projection};
\draw[dashed] (2, 0) -- (2, 2);
\draw[dashed] (2, 0) -- (2.3, 2);
\draw[dashed] (2.3, 0) -- (2.3, 2);
\draw[dashed] (3.7, 0) -- (3.7, 2);
\draw[dashed] (4, 0) -- (3.7, 2);
\draw[dashed] (4, 0) -- (4, 2);

\usetikzlibrary{decorations.pathreplacing}

\draw[decorate,decoration={brace,amplitude=5pt,raise=0pt}] (2,2) -- (2.3,2);
\node[midway,xshift=2.15cm,yshift=2.4cm] {$\xi_3$};

\draw[decorate,decoration={brace,amplitude=5pt,raise=0pt}] (3.7,2) -- (4,2);
\node[midway,xshift=3.85cm,yshift=2.4cm] {$\xi_3$};

\draw[decorate,decoration={brace,amplitude=5pt,mirror,raise=0pt}] (2,0) -- (4,0);
\node[midway,xshift=3cm,yshift=-0.4cm] {$\frac{\epsilon\delta_\kappa}{2n}$};

\node at (1, 2) [left] {$\|H\|$};
\node at (1.5, 2) [above] {$\|H\u\|$};
\node at (2, 2) [left] {};
\node at (2.3, 2) [above] {};

\end{tikzpicture}
\label{fig:eigenvalue}
\end{figure}
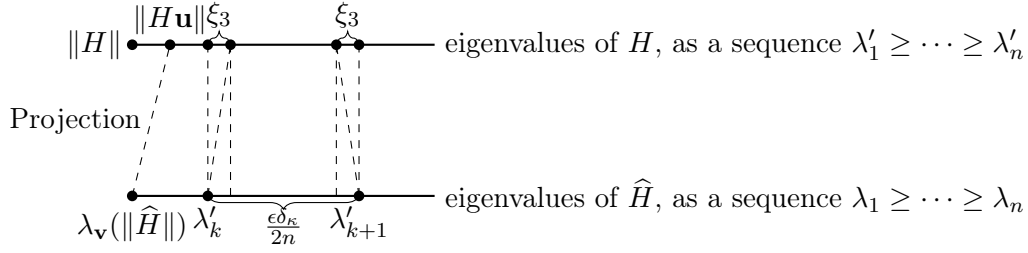

\begin{proof}
Since $\|\widehat H-H\|\le\xi_3$, by Davis-Kahan Theorem (\lem{Davis-Kahan}), we have $k=k'$ and
\[
\sin\big(\Theta(U,V)\big)\le\frac{\xi_3}{\lambda_k-\lambda_{k+1}}=O\left(\frac{n\xi_3}{\epsilon\delta_\kappa}\right),
\]
and hence
\begin{equation}\label{eq:uv_bound}
\|\u-\v\|\le O\left(\frac{n\xi_3}{\epsilon\delta_\kappa}\right).
\end{equation}

For our choice of $\u$, We prove $\|H\|-\|H\u\|\le\frac{\epsilon\delta_\kappa}{2}+2\xi_3$. Our intuition is shown in \fig{eigenvalue}. We find the first pair of eigenvalue $(\lambda_k',\lambda_{k+1}')$ that are far apart (distance longer than $\frac{\epsilon\delta_\kappa}{2n}$), thus we can say $\lambda_1',\ldots,\lambda_k'$ are close enough, and so $\lambda_1,\ldots,\lambda_k$ are close too. Then we can apply Davis-Kahan Theorem to bound the distance between $U$ and $V$. Since $\u\in U$, we can write $\u=\sum_{i=1}^{k}s_i\u_i$ where $\sum_{i=1}^{k}s_i^2=1$. Therefore,
\[
H\u=\sum_{i=1}^ks_iH\u_i=\sum_{i=1}^ks_i\lambda_i\u_i,
\]
and
\[
\|H\u\|=\sqrt{\sum_{i=1}^ks_i^2\lambda_i^2}\ge\lambda_k\sqrt{\sum_{i=1}^ks_i^2}=\lambda_k.
\]
Since for any $i<k$ we have $\lambda_i'-\lambda_{i+1}'\le\frac{\epsilon\delta_\kappa}{2n}$, we can further derive that
\[
\lambda_1'-\lambda_k'\le\sum_{i=1}^{k-1}\lambda_i'-\lambda_{i+1}'\le\frac{(k-1)\epsilon\delta_\kappa}{2n}\le\frac{\epsilon\delta_\kappa}{2}.
\]
From $|\lambda_i-\lambda_i'|\le\xi_3$,
\begin{equation}\label{eq:Hu}
\|H\|-\|H\u\|=\lambda_1-\|H\u\|\le\lambda_1'+\xi_3-(\lambda_k'-\xi_3)\le\frac{\epsilon\delta_\kappa}{2}+2\xi_3.
\end{equation}
\end{proof}

\begin{lemma}\label{lem:difference-colinear}
In the setting above and choose $\u$ as in \lem{choose_eigenvector}, if two vectors $\nabla f(\x+\mu \v)+k_2\v_2$ and $\nabla f(\x)+k_1\v_1+\eta\hat H\cdot\mu \v$ are colinear, then
$\eta=\|H\|(1+O(\frac{\sqrt{\mu}}{\tan\beta})+O(\frac{n\xi_3}{\epsilon\delta_\kappa}))$.
\end{lemma}

\begin{proof}
Note that
\begin{align*}
    \nabla f(\x+\mu \v)&=\nabla f(\x)+\frac{1}{2}\nabla^2f(\x)\cdot\mu \v+O(\mu^2).
\end{align*}
From the observation in \fig{triangle}, if there exist $\eta'$ such that two vectors $\nabla f(\x+\mu \v)+k_2\v_2$ and $\nabla f(\x)+k_1\v_1+\eta' H\cdot\mu \u$ are colinear, then

\begin{align}\label{eq:two_ratios}
    \eta' &= \frac{|BD|}{|BE|}=\frac{|CD|}{|CE|}\cdot\frac{\sin \angle BCD}{\sin \angle BCE}=\frac{\sin \angle CED}{\sin\angle CDE}\cdot\frac{\sin \angle BCD}{\sin \angle BCE}.
\end{align}
We compute two ratios in \eq{two_ratios} separately
\begin{align}\label{eq:CED/BCE}
    \frac{\sin \angle CED}{\sin \angle BCE}&=\frac{\sin(\angle BCE+O(\sqrt{\mu}))}{\sin\angle  BCE} \notag\\
    &=\frac{\sin\angle BCE\cos O(\sqrt{\mu})+\cos \angle BCE\sin O(\sqrt{\mu})}{\sin \angle BCE} \notag\\
    &=\cos O(\sqrt{\mu})+\frac{O(\sqrt{\mu})}{\tan \angle BCE}.
\end{align}
We can then compute the angle
\begin{align}
    \< \nabla^2f(\x)\cdot\mu \u, \nabla^2f(\x)\cdot\mu \u+O(\mu^2)\>&=\arccos\frac{(\lambda_\u\cdot\mu)^2+(\lambda_\u\cdot\mu+O(\mu^2))^2-O(\mu^4)}{2\lambda_\u\cdot\mu(\lambda_\u\cdot\mu+O(\mu^2))} \notag \\
    &=\arccos\frac{2(\lambda_\u\cdot\mu)^2+O(\mu^3)}{2(\lambda_\u\cdot\mu)^2}\left(1-O
    \left(\frac{\mu^2}{\lambda_\u\mu}\right)\right) \notag \\
    &=\arccos(1-O(\mu)) \notag \\
    &=O(\sqrt{\mu}). \label{eq:figure-1}
\end{align}
So there exists an angle $\zeta$ such that $\angle BEC=\frac{\pi-O(\sqrt{\mu})}{2}+\zeta$, $\angle BCE=\frac{\pi-O(\sqrt{\mu})}{2}-\zeta$. From
\begin{align*}
    \sin(\frac{\pi-O(\sqrt{\mu})}{2}+\zeta)=\cos(\frac{O(\sqrt{\mu})}{2}-\zeta) \quad \text{and} \quad \cos x=1-\frac{x^2}{2}+O(x^4),
\end{align*}
we get
\begin{align*}
    &\sin \angle BEC=1-\frac{1}{2}(\frac{O(\sqrt{\mu})}{2}-\zeta)^2+O((\sqrt{\mu}+\zeta)^4), \\
    &\sin\angle  BCE=1-\frac{1}{2}(\frac{O(\sqrt{\mu})}{2}+\zeta)^2+O((\sqrt{\mu}+\zeta)^4).
\end{align*}
Therefore,
\begin{align*}
    \frac{\sin \angle BEC}{\sin \angle BCE}&=\frac{1-\frac{1}{2}(\frac{O(\sqrt{\mu})}{2}-\zeta)^2+O((\sqrt{\mu}+\zeta)^4)}{1-\frac{1}{2}(\frac{O(\sqrt{\mu})}{2}+\zeta)^2+O((\sqrt{\mu}+\zeta)^4)} \\
    &=1+\frac{1}{2}\left(\left(\frac{O(\sqrt{\mu})}{2}+\zeta\right)^2-\left(\frac{O(\sqrt{\mu})}{2}-\zeta\right)^2\right)+O(\mu^{1.5}, \zeta^3) \\
    &=1+O(\sqrt{\mu}\zeta)+O(\mu^{1.5}, \zeta^3).
\end{align*}
On the other hand, by the Law of Sines,
\begin{align*}
    \frac{\sin \angle BEC}{\sin \angle BCE}&=\frac{\|\nabla^2f(\x)\cdot\mu u+O(\mu^2)\|}{\|\nabla^2f(\x)\cdot\mu u\|}=1+O(\mu).
\end{align*}
Comparing two equations with the same left-hand side, we know that: $\zeta=O(\sqrt{\mu})$, i.e. $\angle  BCE=\frac{\pi}{2}+O(\sqrt{\mu})$. Naturally, we have
\begin{align*}
    \tan \angle BCE=O(1/\sqrt{\mu}).
\end{align*} 
Plugging in \eq{CED/BCE}, we obtain
\begin{align*}
    \frac{\sin \angle CED}{\sin\angle BCE}=1+O(\sqrt{\mu}).
\end{align*} 
Denote $\< \nabla f(\x), \nabla f(\x+\mu \v)\>=\alpha$. Applying the Law of Cosines, we can compute $\alpha$ as
\begin{equation*}
    \cos\alpha=\frac{\|\nabla f(\x)\|^2+\|\nabla f(\x+\mu \v)\|^2-\|\nabla^2f(\x)\cdot\mu \v+O(\mu^2)\|^2}{2\cdot\|\nabla f(\x)\|\cdot\|\nabla f(\x+\mu \v)\|}.
\end{equation*}
So we have
\begin{equation}\label{eq:sin-ratio-1}
\alpha=O(\sqrt{\mu}).
\end{equation}
Denote $\< \nabla f(\x), \u\>=\theta$. The other ratio $\frac{\sin\angle BCD}{\sin \angle CDE}$ in \eq{two_ratios} can be computed as
\begin{align}
    \frac{\sin \angle BCD}{\sin \angle CDE}&=\frac{\sin\theta\cos O(\sqrt{\mu})+\cos\theta\sin O(\sqrt{\mu})}{\sin\theta\cos\alpha-\cos\theta\sin\alpha} \notag\\
    &=\frac{1+O(\mu)+\cot\theta\cdot O(\sqrt{\mu})}{1+O(\mu)-\cot\theta\cdot O(\sqrt{\mu})} \notag\\
    &=1+\cot\theta\cdot O(\sqrt{\mu}). \label{eq:sin-ratio-2}
\end{align}
Combining \eq{sin-ratio-1} and \eq{sin-ratio-2}, we have
\begin{align*}
    \eta'&=\frac{\sin \angle BEC}{\sin\angle BCE}\cdot \frac{\sin \angle BCD}{\sin\angle CDE}
    =(1+O(\mu))(1+\cot\theta\cdot O(\sqrt{\mu}))
    =1+O\left(\frac{\sqrt{\mu}}{\tan\theta}\right).
\end{align*}
Since $\|u-v\|\le O(\frac{n\xi_3}{\epsilon\delta_\kappa})$, we can further derive that
\begin{align*}
    \frac{\eta}{\|H\|}
    &=(1+O(\mu))(1+\cot(\theta-O\Big(\frac{n\xi_3}{\epsilon\delta_\kappa}\Big))\cdot O(\sqrt{\mu})) \\
    &=1+O\left(\frac{\sqrt{\mu}}{\tan(\theta-O(\frac{n\xi_3}{\epsilon\delta_\kappa}))}\right)\\
    &=1+O\left(\frac{\sqrt{\mu}}{\tan\beta}\right)+O\left(\frac{n\xi_3}{\epsilon\delta_\kappa}\right).
\end{align*}
\end{proof}

\subsection{Proof of \thm{gradient-Hessian-ratio}}
Here we give a complete proof of our result on computing the ratio of the Hessian and the gradient norm.

\begin{proof}[Proof of \thm{gradient-Hessian-ratio}]
By \lem{choose_eigenvector} we have \eq{estimate_norm}. Given \lin{HE_in_ratio} and \lin{GE_in_ratio} in \algo{Comparison-ratio}, by \thm{main_app} and \thm{Comparison-GDE} we have
\[
\left\|\hat H-\frac{H}{\|H\|}\right\|\le \xi_3,\qquad\hat \g_1=\frac{\g+k_1\v_1}{\|\g+k_1\v_1\|},\qquad \|k_1\v_1\|\le O(\xi_1)\|\g\|.
\]
(Analogous notation holds for $\hat \g_2$ with error $k_2\v_2$.) Observe that
\begin{align*}
\left\|\frac{H}{\|H\|}\cdot\mu \u-\hat{H}\cdot\mu \v\right\|
&\leq
\mu\left(
\left\|\frac{H}{\|H\|}\cdot \u-\frac{H}{\|H\|}\cdot \v\right\|
+
\left\|\frac{H}{\|H\|}\cdot \v-\hat{H}\cdot \v\right\|
\right) \\
&\leq
\mu\left(
\|\u-\v\|+\left\|\frac{H}{\|H\|}-\hat{H}\right\|
\right).
\end{align*}
By \eq{uv_bound}, we have $\|\u-\v\|\le O(\xi_3/\epsilon\delta_\kappa)$. Therefore, we have
\begin{equation}\label{eq:inc_mismatch}
\left\|\frac{H}{\|H\|}\cdot\mu \u-\hat{H}\cdot\mu \v\right\|
\le
O\left(\frac{\mu n\xi_3}{\epsilon\delta_\kappa}\right).
\end{equation}
and
\begin{equation}\label{eq:figure-2}
\left\|H\cdot\mu \u-\|H\|\cdot\hat H\cdot\mu \v\right\|
\le
\|H\|\cdot O\left(\frac{\mu n\xi_3}{\epsilon\delta_\kappa}\right).
\end{equation}
The geometric calculation in \lem{difference-colinear} gives
$\eta/\|H\|=1+O(\sqrt{\mu}/\tan\theta)$, where $\theta=\angle(\g,\u)$. Due to \eq{figure-2}, replacing the exact increment direction $H\mu \u$ by the estimated increment direction $\|H\|\hat H\mu \v$ perturbs the relevant triangle angles by at most $O(\frac{n\xi_3}{\epsilon\delta_\kappa})$, hence we can write
\begin{equation}\label{eq:eta_ratio}
\frac{\eta}{\|H\|}
=
1+O\!\left(\frac{\sqrt{\mu}}{\tan(\theta-O(\frac{n\xi_3}{\epsilon\delta_\kappa}))}\right)
=
1+O\!\left(\frac{\sqrt{\mu}}{\tan\beta}\right)+O\left(\frac{ n\xi_3}{\epsilon\delta_\kappa}\right),
\end{equation}
using $\theta\ge \beta$ and $\frac{\mu n\xi_3}{\epsilon\delta_\kappa}$ sufficiently small. Furthermore, by \lem{vector-norm-1} and \lem{eta_prime_enters},
\begin{equation}\label{eq:Delta_gamma_norm_ratio}
\hat\kappa=\dfrac{\sqrt{1-\inner{\hat{\g}_1}{\hat{\g}_2}^2}} {\mu\,\sqrt{1-\inner{\hat{\g}_2}{\v}^2}}=\frac{\eta\cdot\|\hat H\cdot\mu \v\|}{\mu\|\g+k_1\v_1\|}.
\end{equation}
By Taylor expansion,
\begin{equation}\label{eq:grad_diff_u}
\nabla f(\x+\mu \u)-\nabla f(\x)=H\cdot\mu \u+O(\mu^2),
\end{equation}
so by \eq{Hu} we have
\begin{align}\label{eq:hatgamma_expand}
\frac{1}{\mu}\frac{\|\nabla f(\x+\mu \u)-\nabla f(\x)\|}{\|\g\|}
&= \frac{\|H\cdot\mu \u+O(\mu^2)\|}{\mu\|\g\|}\notag\\
&= \frac{\|H \u\|}{\|\g\|}+O\!\left(\frac{\mu}{\|\g\|}\right)\notag\\
&= \frac{\|H\u\|}{\|\g\|}+O\left(\frac{\mu}{\|\g\|}\right)\notag\\
&= \frac{\|H\|}{\|\g\|}+O\left(\frac{\epsilon\delta_\kappa+\xi_3+\mu}{\|\g\|}\right).
\end{align}
Combine \eq{Delta_gamma_norm_ratio} and \eq{hatgamma_expand},
\begin{align}
\left|\hat\kappa-\frac{1}{\mu}\frac{\|\nabla f(\x+\mu \u)-\nabla f(\x)\|}{\|\g\|}\right|
&=
\left|
\frac{\eta\|\hat H\mu \v\|}{\mu\|\g+k_1\v_1\|}
-
\frac{\|H\mu \u+O(\mu^2)\|}{\mu\|\g\|}
\right| \notag\\
&\le
\underbrace{
\left|
\frac{\eta\|\hat H\mu \v\|}{\mu\|\g+k_1\v_1\|}
-
\frac{\|H\|\|\hat H\mu \v\|}{\mu\|\g+k_1\v_1\|}
\right|}_{T_1}
+
\underbrace{
\left|
\frac{\|H\|\|\hat H\mu \v\|}{\mu\|\g+k_1\v_1\|}
-
\frac{\|H\|\|\hat H\mu \v\|}{\mu\|\g\|}
\right|}_{T_2} \notag\\
&\qquad+
\underbrace{
\left|
\frac{\|H\|\|\hat H\mu \v\|}{\mu\|\g\|}
-
\frac{\|H\mu \u\|}{\mu\|\g\|}
\right|}_{T_3}
+
\underbrace{
\left|\frac{\|H\mu \u\|}{\mu\|\g\|}-\frac{\|H\mu \u+O(\mu^2)\|}{\mu\|\g\|}\right|}_{T_4}.
\label{eq:mediator_split}
\end{align}
Using $\|\hat H\mu \v\|\le \mu\|\hat H\|\le \mu(1+\xi_3)\le 2\mu$ and $\|\g+k_1\v_1\|\ge \|\g\|-\|k_1\v_1\|\ge (1-O(\xi_1))\|\g\|\ge \tfrac12\|\g\|$,
\begin{equation}\label{eq:T1_bound}
T_1
=
\frac{|\eta-\|H\||\cdot\|\hat H\mu \v\|}{\mu\|\g+k_1\v_1\|}
\le
\frac{C}{\|\g\|}\,|\eta-\|H\||.
\end{equation}
By \eq{eta_ratio}, $|\eta-\|H\||\le \|H\|\left(O(\frac{\sqrt{\mu}}{\tan\beta})+O(\frac{n\xi_3}{\epsilon\delta_\kappa})\right)$. Hence
\begin{equation}\label{eq:T1_bound_final}
T_1
\le
\frac{\|H\|}{\|\g\|}\left(O\!\left(\frac{\sqrt{\mu}}{\tan\beta}\right)+O\left(\frac{n\xi_3}{\epsilon\delta_\kappa}\right)\right).
\end{equation}
Using the inequality
$\left|\frac{1}{\|\g+k_1\v_1\|}-\frac{1}{\|\g\|}\right|
=\frac{|\|\g\|-\|\g+k_1\v_1\||}{\|\g\|\|\g+k_1\v_1\|}
\le \frac{\|k_1\v_1\|}{\|\g\|\|\g+k_1\v_1\|}$,
we obtain
\begin{equation}\label{eq:T2_bound}
T_2
\le
\frac{\|H\|\|\hat H\mu \v\|}{\mu}\cdot
\frac{\|k_1\v_1\|}{\|\g\|\|\g+k_1\v_1\|}
\le
\frac{\|H\|}{\|\g\|}\cdot O(\xi_1).
\end{equation}
Moreover,
\[
\big|\|H\|\|\hat H\mu \v\|-\|H\mu \u\|\big|
=
\big|\|H\|\mu\lambda_\v-\mu\|H\u\|\big|
\le
\|H\|\cdot O(\mu (\epsilon\delta_\kappa+\xi_3)).
\]
which gives
\begin{equation}\label{eq:T3_bound}
T_3
\le
\frac{\|H\|\cdot O(\mu (\epsilon\delta_\kappa+\xi_3))}{\mu\|\g\|}
=
\frac{\|H\|}{\|\g\|}\cdot O(\epsilon\delta_\kappa+\xi_3).
\end{equation}
As for $T_4$, from \eq{grad_diff_u} we have $\|O(\mu^2)\|\le C\mu^2$ and thus
\begin{equation}\label{eq:T4_bound}
T_4
\le
\frac{\|O(\mu^2)\|}{\mu\|\g\|}
\le
O\!\left(\frac{\mu}{\|\g\|}\right).
\end{equation}
Plugging \eq{T1_bound_final}--\eq{T4_bound} into \eq{mediator_split} yields
\begin{equation}\label{eq:gamma_hatgamma}
\left|\hat\kappa-\frac{1}{\mu}\frac{\|\nabla f(\x+\mu \u)-\nabla f(\x)\|}{\|\g\|}\right|
\le
\frac{\|H\|}{\|\g\|}
\left(
O\!\left(\frac{\sqrt{\mu}}{\tan\beta}\right)
+O(\frac{n\xi_3}{\epsilon\delta_\kappa})
+O(\xi_1)
+O(\epsilon\delta_\kappa+\xi_3)
\right)
+
O\!\left(\frac{\mu}{\|\g\|}\right),
\end{equation}
Combine \eq{hatgamma_expand} and \eq{gamma_hatgamma}, we have
\begin{align}\label{eq:ratio_bound}
\left|\hat\kappa-\frac{\|H\|}{\|\g\|}\right| 
&\le 
\frac{\|H\|}{\|\g\|}
\left(
O\!\left(\frac{\sqrt{\mu}}{\tan\beta}\right)
+O(\frac{n\xi_3}{\epsilon\delta_\kappa})
+O(\xi_1)
+O(\epsilon\delta_\kappa+\xi_3)
\right)
+
O\!\left(\frac{\mu}{\|\g\|}\right)\notag\\
&\le O\left(\frac{\sqrt{\mu}}{\epsilon}
+\frac{n\xi_3}{\epsilon^2\delta_\kappa}
+\frac{\xi_1+\xi_3}{\epsilon}
+\delta_\kappa\right).
\end{align}
By our choice of parameters in \algo{Comparison-ratio} and using the Lipschitzness, the right-hand side of \eq{ratio_bound} is at most $O(\delta_\kappa)$, which completes our proof.
\end{proof}


\section{Finding Stationary Points}
\label{sec:best_of_all}
In this section, we combine the tools developed in the previous sections to construct \algo{best-of-all}, which outputs a list of points that includes at least one $\epsilon$-second-order stationary point, a notion that generalizes and is strictly stronger than an $\epsilon$-stationary point.
\begin{definition}
\label{defn:sosp}
A point $\x\in\R^n$ is an $\epsilon$-second-order stationary point of $f$ if
\[
\|\nabla f(\x)\|\le \epsilon
\qquad\text{and}\qquad
\lambda_{\min}\!\left(\nabla^2 f(\x)\right)\ge -\sqrt{L_2\epsilon}.
\]
\end{definition}

\begin{algorithm2e}[h]
\caption{\texttt{ComparisonTR}($\epsilon$, $\xi_1$, $\xi_2$, $\xi_3$, $T$): Finding stationary point}
\label{algo:best-of-all}
\LinesNumbered
\DontPrintSemicolon
\KwIn{Initial point $\x_0$, precision $\epsilon$, tolerances $(\xi_1,\xi_2,\xi_3)=(1/100,1/100,1/100n)$, number of iterations $T$}
\KwOut{A list of points containing at least one $\epsilon$-second order stationary point.}

$r\gets \sqrt{\epsilon/L_2}$.\;
$\delta_\kappa\gets c_0\sqrt{\frac{\epsilon}{L_2}}$.\;
\For{$t=0,\ldots,T-1$}{
    $\hat \g\gets \texttt{ComparisonGE}(\x_t;\xi_1)$.\;
    $\hat H\gets \texttt{ComparisonHE}(\x_t;\xi_3)$.\label{lin:estimate_hessian}\;
    $\hat\kappa \gets \texttt{ComparisonRatio}(\x_t;\delta_\kappa)$.\;\label{lin:ratio_in_FOSP} 
    Compute a unit eigenvector $\v$ of $\hat H$ with eigenvalue of largest magnitude.\;

    $\x^{(1)}\gets \x_t-r\hat \g$.
    
    Find $l^\star\in[-r,r]$ approximately minimizing $f(\x_t+l\v)$.\;
    $\x^{(2)} \gets \x_t+l^\star \v$.\;
    
    Define \begin{align*}
    m_\x(\p)=\inner{\hat \g}{\p}+\frac{\hat\kappa}{2}\,\p^\top \hat H\p, \qquad \|\p\|\le r.
    \end{align*}
    Let $\p^\star\in\arg\min_{\|\p\|\le r} m_\x(\p)$.\;
    $\x^{(3)} \gets \x+\p^\star$.\label{lin:trust-region}\;
    
    $\x_{t+1}=\arg\min\{f(\x^{(1)}), f(\x^{(2)}), f(\x^{(3)}), f(\x_t)\}$.\;\label{lin:must_candidate}
}

\Return{$\{\x_0,\ldots,\x_T\}$}.\;
\end{algorithm2e}

\begin{theorem}\label{thm:Comparison-Newton}
Given $f\colon\R^n\to\R$ has $L_1$-Lipschitz gradient and $L_2$-Lipschitz Hessian. Denote $\x^\star=\arg\min f(\x)$. Given $\x_0\in\R^n$ satisfying $f(\x_0)-f(\x^\star)\leq \Delta$, with success probability at least $2/3$, \algo{best-of-all} outputs a list of $O(1/\epsilon^{1.5})$ points, in which at least one point is an $\epsilon$-second order stationary point, using $T=O(\frac{f(\x_0)-f(\x^\star)}{C'\epsilon^{1.5}})$ iterations. The overall query complexity is $\tilde{O}(\frac{\Delta\sqrt{L_2}n^2}{\epsilon^{1.5}})$. 
\end{theorem}

If in some steps $\nabla f(\x_t)\le \epsilon$ and $\lambda_{\min}(\nabla^2f(\x_t))\ge-\sqrt{L_2\epsilon}$, we have already visited an $\epsilon$-second order stationary point. Otherwise, at any iterate $\x_t$ with $\|\nabla f(\x_t)\|\ge\epsilon$, \algo{best-of-all} finds the next iterate $\x_{t+1}$ that decreases the function value by $\Omega(\epsilon^{3/2})$. If after limited steps we cannot visit a stationary point, the function value will decrease more than $f(\x_0)-f(\x^\star)$, which is a contradiction. The iteration rule constructs three candidate steps (gradient step, trust-region step on a normalized quadratic model, and rank-one line-search step), and then selects the candidate with minimal function value by comparison. If the Hessian is rank-one or the eigenvector with the largest absolute eigenvalue has significant overlap with the gradient, line search performs well. Else if the Hessian has too small norm, or positive/negative definite, a normalized gradient descent step decreases $\Omega(\epsilon^{1.5})$. Else, in all the other cases, trust-region step guarantees the decrease.

\begin{lemma}
\label{lem:best-of-all}
For any iterate $\x_t$ in \algo{best-of-all} that satisfies: (i) $\|\nabla f(\x_t)\|\ge \epsilon$, or  (ii) $\|\nabla f(\x_t)\|< \epsilon$ and $\lambda_{\min}(\nabla^2f(\x_t))\le -\sqrt{L_2\epsilon}$, the next iterate $\x_{t+1}$ satisfies
\begin{equation}\label{eq:best_of_all_descent}
f(\x_{t+1})\le f(\x_t)-C\epsilon^{1.5}
\end{equation}
for some absolute constant $C>0$.
\end{lemma}

\subsection{Descent of each candidate}
For any iterate $\x_t$ with $\|\nabla f(\x_t)\|\ge \epsilon$, let $H:=\nabla^2f(\x_t)$ and $\g:=\nabla f(\x_t)$. \algo{best-of-all} builds three candidates within the ball:
\begin{itemize}
\item[\emph{(i)}] normalized gradient descent along $-\hat \g$ with step size in $[0,r]$,
\item[\emph{(ii)}] a line search along an eigenvector $\v$ corresponding to the eigenvalue with the largest magnitude on $[-r,r]$, 
\item[\emph{(iii)}] a trust-region minimizer of the approximate quadratic $m_{\x_t}(\p)=\inner{\hat \g}{\p}+\frac{\hat\kappa}{2}\p^\top \hat H \p$ with $\|\p\|\le r$.
\end{itemize}

\begin{lemma}[Descent of Candidate 1]\label{lem:cand1_descent}
Given $f\colon\R^n\to\R$ has $L_1$-Lipschitz gradient and $L_2$-Lipschitz Hessian. For any iterate $\x_t\in\R^n$ with $\|\g\|\ge \epsilon$, let $\hat \g$ be the output of $\texttt{ComparisonGE}(\x;\xi_1)$ with
\begin{equation}\label{eq:cand1_hatg_acc}
\left\|\hat \g-\frac{\g}{\|\g\|}\right\|\le \xi_1,\qquad \xi_1\le \frac{1}{100}.
\end{equation}
Assume that at least one of the following two conditions holds:
\begin{enumerate}
\item[\textup{(i)}] (\emph{small Hessian norm}) $\|H\|\le \sqrt{L_2\epsilon}$;
\item[\textup{(ii)}] (\emph{$H$ is definite}) $ H\succeq \0$ or $H\preceq \0$, and $\|H\|\ge \sqrt{L_2\epsilon}$;
\end{enumerate}
Then there exists an absolute constant $C>0$ such that
\begin{equation}\label{eq:cand1_true_decrease_rigorous}
f(\x^{(1)})\le f(\x_t)-C\,\epsilon^{3/2}.
\end{equation}
\end{lemma}

\begin{proof}
We use the cubic Taylor remainder bound for any $\p\in\R^n$,
\begin{equation}\label{eq:cand1_taylor}
f(\x_t+\p)\le f(\x_t)+\g^\top \p+\frac12 \p^\top H \p+\frac{L_2}{6}\|\p\|^3.
\end{equation}
Denote $\p_1=-r\hat \g$. We will bound the linear and quadratic terms for $\p_1$ and then absorb the remainder term.

Let $\bar \g:=\g/\|\g\|$. Since $\|\hat \g-\bar \g\|\le\xi_1$ and both are unit vectors,
\[
\inner{\bar \g}{\hat \g}
=
1-\frac12\|\bar \g-\hat \g\|^2
\ge 1-\frac12\xi_1^2
\ge 1-\xi_1.
\]
Therefore
\begin{equation}\label{eq:cand1_lin}
\g^\top \p_1 = -r\,\inner{\g}{\hat \g}
= -r\,\|\g\|\,\inner{\bar \g}{\hat \g}
\le -(1-\xi_1)\,r\,\|\g\|.
\end{equation}
Since $\|\p_1\|=r$ and $\|H\|$ denotes the spectral norm,
\begin{equation}\label{eq:cand1_quad_worst}
\frac12 \p_1^\top H \p_1 \le \frac12 \|H\|\,\|\p_1\|^2 = \frac12\|H\|\,r^2.
\end{equation}
Plugging \eq{cand1_lin} and \eq{cand1_quad_worst} into \eq{cand1_taylor} with $\p=\p_1$ yields
\begin{equation}\label{eq:cand1_master}
f(\x^{(1)})-f(\x_t)
\le
-(1-\xi_1)r\|\g\|+\frac12\|H\|r^2+\frac{L_2}{6}r^3.
\end{equation}
Using $\|\g\|\ge\epsilon$ and $r=\sqrt{\epsilon/L_2}$ gives
\[
r\|\g\| \ge \,\frac{\epsilon^{3/2}}{\sqrt{L_2}},
\qquad
r^2 = \frac{\epsilon}{L_2},
\qquad
\frac{L_2}{6}r^3 = \frac{1}{6}\frac{\epsilon^{3/2}}{\sqrt{L_2}}.
\]
Hence \eq{cand1_master} becomes
\begin{equation}\label{eq:cand1_master_scaled}
f(\x^{(1)})-f(\x_t)
\le
-\Big((1-\xi_1)-\frac{1}{6}\Big)\frac{\epsilon^{3/2}}{\sqrt{L_2}}
+\frac12\|H\|\cdot\frac{\epsilon}{L_2}.
\end{equation}
We now argue that in each of the three cases (i)--(iii), the positive quadratic term is controlled so that the right hand side of \eq{cand1_master_scaled} is less than $-C\,\epsilon^{3/2}$ after choosing $c$ as a sufficiently small constant.

In Case (i), $\|H\|\le \sqrt{L_2\epsilon}$. Then the quadratic term in \eq{cand1_master_scaled} is bounded by
\[
\frac12\|H\|\cdot\frac{\epsilon}{L_2}
\le
\frac12 \sqrt{L_2\epsilon}\cdot\frac{\epsilon}{L_2}
=
\frac12\frac{\epsilon^{3/2}}{\sqrt{L_2}}.
\]
Therefore
\[
f(\x^{(1)})-f(\x_t)
\le
-\left((1-\xi_1)-\frac{1}{6}-\frac12 \right)\frac{\epsilon^{3/2}}{\sqrt{L_2}}.
\]
So we have
\[
f(\x^{(1)})\le f(\x_t)-C\,\epsilon^{3/2}
\]
for an absolute constant $C>0$.

In Case (ii), $H\preceq \0$ or $H\succeq \0$. By $L_2$-Lipschitzness, for any $p$,
\[
f(\x_t+\p)\le f(\x_t)+\g^\top \p+\frac12 \p^\top H \p+\frac{L_2}{6}\|\p\|^3.
\]
Apply this with $\p=\p_1=-r\hat \g$ to obtain
\begin{equation}\label{eq:l1_smooth_step}
f(\x^{(1)})-f(\x_t)\le r\g^\top \hat\g+\frac{r^2}{2}\hat\g^\top H\hat\g+\frac{L_2}{6}|r|^3.
\end{equation}

Let $\bar \g:=\g/\|\g\|$. Since $\|\hat \g-\bar \g\|\le \xi_1$ and both are unit vectors,
\[
\inner{\bar \g}{\hat \g}
=1-\frac12\|\bar \g-\hat \g\|^2
\ge 1-\frac12\xi_1^2
\ge 1-\xi_1.
\]
Therefore
\[
\g^\top \p_1=-r\,\inner{\g}{\hat \g}
=-r\|\g\|\,\inner{\bar \g}{\hat \g}
\le -(1-\xi_1)\,r\|\g\|.
\]

If $H\preceq \0$, we have 
\[
\frac{1}{2}\p_1^\top H\p_1\le  0.
\]
so
\[
f(\x^{(1)})- f(\x_t)\le\g^\top \p_1+\frac{L_2}{6}|r|^3\le-\epsilon\left((1-\xi_1)-\frac{1}{6}\right)r=-\Omega(\epsilon^{3/2}).
\]

Else if $ H\succeq \0$, consider the 1-dimension function $f_\star$ on the direction $\hat\g$. By the Lipschitzness of $H$, for any $l\in[0,r]$, we have 
\[
f_\star''(\x_t-l\hat\g)\ge 0.
\] 
It tells us that 
\[
f_\star'(\x_t-l\hat\g)\ge \epsilon.
\]
So the function value descent 
\[
f(\x^{(1)})- f(\x_t)\le -r\epsilon=-\Omega(\epsilon^{3/2}).
\]
\end{proof}

\begin{lemma}[Descent of Candidate 2]\label{lem:cand2_descent}
Let $\x^{(2)}=\x_t+l^\star \v$ where $l^\star\in[-r,r]$ minimizes $f(\x_t+t \v)$ over $[-r,r]$.
Then either $\|\g\|\le \epsilon$, or
\begin{equation}\label{eq:cand2_true}
f(\x^{(2)})-f(\x_t)\le -c_2\,\epsilon^{3/2}
\end{equation}
for an absolute constant $c_2>0$.
\end{lemma}

\begin{proof}
Consider $\phi(t)=f(\x_t+l \v)$ and expand around $l=0$, we have
\[
\phi(l)=\phi(0)+l\inner{\g}{\v}+\frac12 l^2\,\v^\top H \v + O(L_2|l|^3).
\]
For $|\inner{\g}{\v}|\ge \frac12\|\g\|$, choosing $l=-r\,\mathrm{sign}(\inner{\g}{\v})$ yields
\[
\phi(l)-\phi(0)\le -r\cdot \tfrac12\|\g\| + O(L_2 r^3)
\le -\Omega(\epsilon^{3/2})
\]
since $\|\g\|\ge \epsilon$ and $L_2 r^3=\epsilon^{3/2}/\sqrt{L_2}$.
If instead $|\inner{\g}{\v}|<\frac12\|\g\|$, then $\g$ has a component orthogonal to $\v$ of magnitude at least $\frac{\sqrt{3}}{2}\|\g\|$, and Candidate~1 already provides a decrease of order $r\|\g\|$ up to quadratic/cubic corrections; therefore the best-of-all selection ensures the stated decrease.
\end{proof}

\begin{lemma}[Descent of Candidate 3]\label{lem:cand3_descent}
Let $\p^\star$ minimize the normalized quadratic model
\[
m_{\x_t}(\p)=\inner{\g}{\p}+\frac{\hat\kappa}{2}\,\p^\top \hat H\,\p
\quad\text{subject to }\|\p\|\le r,
\]
Then either
\begin{equation}\label{eq:cand3_stationary}
\|\nabla m_{\x_t}(\p^\star)\|\le \epsilon
\end{equation}
or
\begin{equation}\label{eq:cand3_model_descent}
f(\x^{(3)})-f(\x_t)=m_{\x_t}(\p^\star)-m_{\x_t}(\0)\le -\frac12 r\|\nabla m_{\x_t}(\p^\star)\|.
\end{equation}
In particular, in the non-stationary case \eq{cand3_model_descent} implies
$m_{\x_t}(\p^\star)-m_{\x_t}(\0)\le -\Omega(r\epsilon)=-\Omega(\epsilon^{3/2})$.
\end{lemma}

\begin{proof}
This is a direct corollary of \lem{tr_quadratic}, taking $q=m_{\x_t}$, $b=\g$
and $A=\hat\kappa\hat H$.
\end{proof}

Next, we prove that \algo{best-of-all} can go further to visit an ($\epsilon$, $\sqrt{\epsilon}$)-second order stationary point. Our intuition is to construct a negative curvature candidate and prove that the trust region candidate is better than or equal to it. 
\begin{lemma}\label{lem:cand4_lem}
For any iterate $\x_t\in\R^n$, suppose that $\v_{\min}$ is the unit eigenvector corresponding to the minimum eigenvalue of $H$, denoted as $\lambda_{\min}$. If $\inner{\v_{\min}}{\g}\ge 0$, set $\v_{\min}=-\v_{\min}$. Let $\x^{(4)}=\x_t+r\v_{\min}$. If $\|\g\|\le\epsilon$ and $\lambda_{\min}\le-\sqrt{L_2\epsilon}$, then
\begin{equation}\label{eq:cand4}
f(\x^{(4)})-f(\x_t)\le-c_4\epsilon^{3/2}
\end{equation}
for an absolute constant $c_4>0$.
\end{lemma}

\begin{proof}
We use the cubic Taylor remainder bound for any $\p\in\R^n$,
\begin{equation*}
f(\x_t+\p)\le f(\x_t)+\g^\top \p+\frac12 \p^\top H \p+\frac{L_2}{6}\|\p\|^3.
\end{equation*}
Plugging in $\p_4=r\v_{\min}$ we have
\begin{align*}
f(\x^{(4)})&\le f(\x_t)+r\g^\top \v_{\min}+\frac12 r^2 \v_{\min}^\top H \v_{\min}+\frac{L_2}{6}r^3 \\
&= f(\x_t)+\sqrt{\frac{\epsilon}{L_2}}\g^\top \v_{\min}+\frac12 \frac{\epsilon}{L_2} \v_{\min}^\top H \v_{\min}+\frac{L_2}{6}\cdot \frac{\epsilon^{3/2}}{L_2^{3/2}} \\
&\le f(\x_t)+\sqrt{\frac{\epsilon^3}{L_2}}\left(-\frac12+\frac16\right)=f(\x_t)-c_4\epsilon^{3/2}.
\end{align*}
\end{proof}

\begin{lemma}\label{lem:trust_region_better}
For any iterate $\x_t\in\R^n$ and let $\g=\nabla f(\x_t)$ and $H=\nabla^2f(\x_t)$. If $\|\g\|\le\epsilon$ and $\lambda_{\min}\le-\sqrt{L_2\epsilon}$, then
\begin{equation}\label{eq:trust_region_better}
f(\x^{(3)})\le f(\x^{(4)})+c_5\epsilon^{3/2}
\end{equation}
for an absolute constant $c_5>0$.
\end{lemma}

\begin{proof}
Denote $x^{(3)}=\x_t+\p_3$, $x^{(4)}=\x_t+\p_4$. For our choice of $x^{(3)}$, we have
\begin{equation*}
m_{\x_t}(\p_3)\le m_{\x_t}(\p_4),
\end{equation*}
which is equivalent to
\begin{equation}\label{eq:sosp_1}
T_{\x_t}(\x_t+\p_3)\le T_{\x_t}(\x_t+\p_4).
\end{equation}
By \lem{taylor_cubic_best} we have
\begin{equation}\label{eq:sosp_2}
|f(\x_t+\p_3)-T_{\x_t}(\x_t+\p_3)|\le \frac{L_2}{6}\|\p_3\|^3=\frac{L_2}{6}r^3,
\end{equation}
\begin{equation}\label{eq:sosp_3}
|f(\x_t+\p_4)-T_{\x_t}(\x_t+\p_4)|\le \frac{L_2}{6}\|\p_4\|^3=\frac{L_2}{6}r^3.
\end{equation}
Combining \eq{sosp_1}, \eq{sosp_2} and \eq{sosp_3}, we can conclude \eq{trust_region_better}.
\end{proof}

\subsection{Proof of function value descent guaranty}\label{sec:proof_of_descent}
Here, we prove the theoretical guarantee of our \algo{best-of-all} for a large descent when the current point is not a second-order stationary point.
\begin{proof}[Proof of \lem{best-of-all}]
Assume $\|\g(\x_t)\|\ge \epsilon$. Let $\x_{t+1}=\arg\min\{f(\x^{(1)}),f(\x^{(2)}),f(\x^{(3)})\}$ in \algo{best-of-all}. By \lem{taylor_cubic_best}, for each candidate $\p_i=\x^{(i)}-\x_t$ with $\|\p_i\|\le r$,
\begin{equation}\label{eq:model_to_true_transfer}
f(\x^{(i)})-f(\x_t)
\le
\big(T_{\x_t}(\x_t+\p_i)-T_{\x_t}(\x_t)\big) + \frac{L_2}{3}r^3.
\end{equation}

We show that at least one candidate has
$T_{\x_t}(\x_t+\p_i)-T_{\x_t}(\x_t)\le -c' r\epsilon$ for a constant $c'>0$; then since
$\frac{L_2}{3}r^3=\frac{1}{3}\epsilon^{3/2}/\sqrt{L_2}$, for $c$ sufficiently small, we obtain
$f(\x^{(i)})-f(\x_t)\le -\Omega(\epsilon^{3/2})$ and thus the same for $\x_{t+1}$.

If the Hessian is rank-one or $\inner{\g}{\v}\ge\frac12\|\g\|$, $\x^{(2)}$ ensures the $\Omega(\epsilon^{3/2})$ decrease by \lem{cand2_descent}. Else, for cases in \lem{cand1_descent},
\[
T_{\x_t}(\x_t+\p_1)-T_{\x_t}(\x_t)\le -\frac{99}{100}r\|\g\|+\frac12 r^2\|H\|.
\]
Using the ratio $\|H\|/\|\g\|\approx \hat\kappa$,
we have 
\[
\frac12 r^2\|H\|\le c r\|\g\|
\]
at the scale $r=\sqrt{\epsilon/L_2}$, so the linear term dominates for small $\epsilon$.
Hence 
\[
T_{\x_t}(\x_t+\p_1)-T_{\x_t}(\x_t)\le -\Omega(r\|\g\|)\le -\Omega(r\epsilon)=-\Omega(\epsilon^{3/2}).
\]
Then \eq{model_to_true_transfer} yields 
\[
f(\x^{(1)})-f(\x_t)\le -\Omega(\epsilon^{3/2}).
\]

In the case that $H$ is not rank-one, $\|H\|\ge\sqrt{L_2\epsilon}$, and $H$ is not positive or negative definite. Denote the eigenvector of $H$ that corresponding to the largest magnitude eigenvalue as $\v$ and assume that $\inner{\v}{\g}\le\frac12\|\g\|$. 

If \eq{cand3_stationary} fails, then by \lem{cand3_descent}, we have
\[
m_{\x_t}(\p_3)-m_{\x_t}(\0)\le -\frac12 r\|\nabla m_{\x_t}(\p_3)\|.
\]
Since the failure of \eq{cand3_stationary} means $\|\nabla m_{\x_t}(\p_3)\|>\epsilon$, we get 
\[
m_{\x_t}(\p_3)-m_{\x_t}(\0)\le -\frac12 r\epsilon=-\Omega(\epsilon^{3/2}).
\]
By the constant-factor accuracy of $\hat\kappa$ and $\hat H$, the model $m_{\x_t}$ is a constant-factor approximation of the quadratic model $T_{\x_t}$ when $\|\p\|\le r$. Thus 
\[T_{\x_t}(\x_t+\p_3)-T_{\x_t}(\x_t)\le -\Omega(\epsilon^{3/2}),
\]
and \eq{model_to_true_transfer} gives 
\[f(\x^{(3)}-f(\x_t)\le -\Omega(\epsilon^{3/2}).
\]

Assume $\|\g\|\le\epsilon$ and $\lambda_{\min}(H)\le-\sqrt{L_2\epsilon}$. By \lem{trust_region_better} we have 
\[
f(\x_3)-f(\x_t)\le -\Omega(\epsilon^{3/2}).
\]

In all cases, at least one candidate $x^{(i)}$ satisfies $f(\x^{(i)})\le f(\x_t)-C\epsilon^{3/2}$.
Since $\x_{t+1}=\arg\min\{f(\x^{(1)}),f(\x^{(2)}),f(\x^{(3)})\}$, we have $f(\x_{t+1})\le f(\x^{(i)})$ for that
index $i$, proving \eq{best_of_all_descent}.
\end{proof}

\subsection{Proof of \thm{Comparison-Newton}}
Here we give a complete proof of our result main result on finding the stationary points.

\begin{proof}[Proof of \thm{Comparison-Newton}]
We use contradiction to prove the theorem. If the statement is not true and any iterate $\x_t$ is not an $\epsilon$-stationary point, by \lem{best-of-all} we have
\begin{align*}
\E [f(\x_{t+1})-f(\x_t)]\leq -\frac{2C}{3}\epsilon^{1.5},
\end{align*}
where $C$ is the constant in \lem{best-of-all}, since we have $f(\x_{t+1})-f(\x_t)\leq -C\epsilon^{1.5}$ if \algo{normhess} succeeds, which happens with probability at least $2/3$ by \thm{main_app}, and we have $f(\x_{t+1})-f(\x_t)\leq 0$ otherwise. 
Sum from $t=0$ to $T-1$ we get
\begin{align*}
\E [f(\x_T)-f(\x_0)]&=\sum_{t=0}^{T-1}\E [f(\x_{t+1})-f(\x_t)]
=-\sum_{t=0}^{T-1}C'\epsilon^{1.5}
=-TC'\epsilon^{1.5}.
\end{align*}
By Markov's Inequality, 
\begin{align*}
\Pr(f(\x_0)-f(\x_T)\ge \frac{1}{2}TC'\epsilon^{1.5})
\ge \frac{a-\E[f(\x_0)-f(\x_T)]}{a-\frac{1}{2}TC'\epsilon^{1.5}}
=\frac{2}{3}.
\end{align*}

Therefore, with probability at least $2/3$, the total decrease of function value is larger than $f(\x_0)-f(\x^\star)$, which leads to a contradiction. Consequently, in $T$ steps, we have visited at least an $\epsilon$-stationary point. 
\end{proof}

Note that we have used $\texttt{Comparison-GE}$ in \thm{Comparison-GDE} as a subroutine to estimate the direction of the gradient. The query complexity of $\texttt{Comparison-GE}$ is $O(n\log\frac{1}{\epsilon})$, which incurs an $O(n)$ overhead -- intuitively, classical algorithms by comparisons are limited by the fact that we need $\Omega(n)$ comparisons to explore an $n$-dimensional space. An idea that can significantly reduce the query complexity of our algorithm is by replacing $\texttt{Comparison-GE}$ with $\texttt{Comparison-QGE}$ in \thm{Comparison-QGDE}, a quantum algorithm with $\log n$ dependence for gradient estimation. This implies the following corollary.
\begin{corollary}
There exists a quantum algorithm that visits an $\epsilon$-second order stationary point using $O(\frac{\Delta\sqrt{L_2}n}{\epsilon^{1.5}}\log\big(\frac{nL_1L_2}{\epsilon}\big))$ queries to a quantum comparison oracle \eqref{eq:quantum_comparison}.
\end{corollary}


\section*{Acknowledgements}
We thank the anonymous reviewers for their constructive feedback. HW, XT, YZ, and TL were supported by the National Natural Science Foundation of China (Grant Number 62372006).

\newcommand{\arxiv}[1]{arXiv:\href{https://arxiv.org/abs/#1}{\ttfamily{#1}}\?}\newcommand{\arXiv}[1]{arXiv:\href{https://arxiv.org/abs/#1}{\ttfamily{#1}}\?}\def\?#1{\if.#1{}\else#1\fi}

\newpage
\appendix
\onecolumn

\section{Auxiliary Lemmas}
In this appendix, we collect all auxiliary lemmas needed for our proofs.

\subsection{Distance between normalized vectors}
\begin{lemma}\label{lem:dist-norm-vector}
If $\v,\v'\in\R^{n}$ are two vectors such that $\|\v\|\geq\gamma$ and $\|\v-\v'\|\leq\tau$, we have
\begin{align*}
\left\|\frac{\v}{\|\v\|}-\frac{\v'}{\|\v'\|}\right\|\leq\frac{2\tau}{\gamma}.
\end{align*}
\end{lemma}
\begin{proof}
By the triangle inequality, we have
\begin{align*}
\left\|\frac{\v}{\|\v\|}-\frac{\v'}{\|\v'\|}\right\|&\leq \left\|\frac{\v}{\|\v\|}-\frac{\v'}{\|\v\|}\right\|+\left\|\frac{\v'}{\|\v\|}-\frac{\v'}{\|\v'\|}\right\|\\
&=\frac{\|\v-\v'\|}{\|\v\|}+\frac{|\|\v\|-\|\v'\||\|\v'\|}{\|\v\|\|\v'\|}\\
&\leq\frac{\tau}{\gamma}+\frac{\tau}{\gamma}=\frac{2\tau}{\gamma}.
\end{align*}
\end{proof}

\begin{lemma}\label{lem:dist-inner-vector}
If $\v_1,\v_2\in\R^{n}$ are two vectors such that $\|\v_1\|,\|\v_2\|\geq\gamma$, and $\v_1',\v_2'\in\R^{n}$ are another two vectors such that $\|\v_1-\v_1'\|, \|\v_2-\v_2'\|\leq\tau$ where $0<\tau<\gamma$, we have
\begin{align*}
\left|\left\langle\frac{\v_1}{\|\v_1\|},\frac{\v_2}{\|\v_2\|}\right\rangle-\left\langle\frac{\v_1'}{\|\v_1'\|},\frac{\v_2'}{\|\v_2'\|}\right\rangle\right|\leq\frac{6\tau}{\gamma}.
\end{align*}
\end{lemma}
\begin{proof}
By the triangle inequality, we have
\begin{align*}
\left|\left\langle\frac{\v_1}{\|\v_1\|},\frac{\v_2}{\|\v_2\|}\right\rangle-\left\langle\frac{\v_1'}{\|\v_1'\|},\frac{\v_2'}{\|\v_2'\|}\right\rangle\right|&\leq
\left|\left\langle\frac{\v_1}{\|\v_1\|},\frac{\v_2}{\|\v_2\|}\right\rangle-\left\langle\frac{\v_1'}{\|\v_1\|},\frac{\v_2'}{\|\v_2\|}\right\rangle\right| \\
&\quad+\left|\left\langle\frac{\v_1'}{\|\v_1\|},\frac{\v_2'}{\|\v_2\|}\right\rangle-\left\langle\frac{\v_1'}{\|\v_1'\|},\frac{\v_2'}{\|\v_2'\|}\right\rangle\right|.
\end{align*}
On the one hand, by the triangle inequality and the Cauchy-Schwarz inequality,
\begin{align*}
\left|\left\langle\frac{\v_1}{\|\v_1\|},\frac{\v_2}{\|\v_2\|}\right\rangle-\left\langle\frac{\v_1'}{\|\v_1\|},\frac{\v_2'}{\|\v_2\|}\right\rangle\right|
&\leq \frac{1}{\|\v_1\|\|\v_2\|}(\left|\langle\v_1,\v_2\rangle-\langle\v_1,\v_2'\rangle\right|+\left|\langle\v_1,\v_2'\rangle-\langle\v_1',\v_2'\rangle\rangle\right|)\\
&\leq \frac{\|\v_2-\v_2'\|}{\|\v_2\|}+\frac{\|\v_1-\v_1'\|\|\v_2'\|}{\|\v_1\|\|\v_2\|}\\
&\leq \frac{\tau}{\gamma}+\frac{\tau(\gamma+\tau)}{\gamma^{2}}.
\end{align*}
On the other hand, by the Cauchy-Schwarz inequality, $|\langle \v_1',\v_2'\rangle|\leq\|\v_1'\|\|\v_2'\|$, and hence
\begin{align*}
\left|\left\langle\frac{\v_1'}{\|\v_1\|},\frac{\v_2'}{\|\v_2\|}\right\rangle-\left\langle\frac{\v_1'}{\|\v_1'\|},\frac{\v_2'}{\|\v_2'\|}\right\rangle\right|&=|\langle \v_1',\v_2'\rangle|\left|\frac{1}{\|\v_1\|\|\v_2\|}-\frac{1}{\|\v_1'\|\|\v_2'\|}\right|\\
&\leq \left|\frac{\|\v_1'\|\|\v_2'\|}{\|\v_1\|\|\v_2\|}-1\right|\\
&\leq \left(\frac{\gamma+\tau}{\gamma}\right)^{2}-1.
\end{align*}
In all, due to $\tau<\gamma$,
\begin{align*}
\left|\left\langle\frac{\v_1}{\|\v_1\|},\frac{\v_2}{\|\v_2\|}\right\rangle-\left\langle\frac{\v_1'}{\|\v_1'\|},\frac{\v_2'}{\|\v_2'\|}\right\rangle\right|\leq
\frac{\tau}{\gamma}+\frac{\tau(\gamma+\tau)}{\gamma^{2}}+\left(\frac{\gamma+\tau}{\gamma}\right)^{2}-1=\frac{2\tau(2\gamma+\tau)}{\gamma^{2}}\leq\frac{6\tau}{\gamma}.
\end{align*}
\end{proof}

\begin{lemma}
\label{lem:vector-norm-1}
For any nonzero vectors $\v, \g\in\R^{n}$,
\[
\sqrt{\frac{1-\left\langle\frac{\v+\g}{\|\v+\g\|},\frac{\v}{\|\v\|}\right\rangle^{2}}
{1-\left\langle\frac{\v-\g}{\|\v-\g\|},\frac{\v}{\|\v\|}\right\rangle^{2}}}
=\frac{\|\v-\g\|}{\|\v+\g\|}.
\]
\end{lemma}

\begin{proof}
We have
\begin{align*}
\frac{1-\langle\frac{\v+\g}{\|\v+\g\|},\frac{\v}{\|\v\|}\rangle^{2}}{1-\langle\frac{\v-\g}{\|\v-\g\|},\frac{\v}{\|\v\|}\rangle^{2}}\cdot\frac{\|\v+\g\|^{2}}{\|\v-\g\|^{2}}&=\frac{\|\v+\g\|^{2}-\langle \v+\g,\frac{\v}{\|\v\|}\rangle^{2}}{\|\v-\g\|^{2}-\langle \v-\g,\frac{\v}{\|\v\|}\rangle^{2}}\\
&=\frac{\langle\v+\g,\v+\g\rangle-(\|\v\|+\frac{\langle\v,\g\rangle}{\|\v\|})^{2}}{\langle\v-\g,\v-\g\rangle-(\|\v\|-\frac{\langle\v,\g\rangle}{\|\v\|})^{2}}\\
&=\frac{\|\v\|^{2}+\|\g\|^{2}+2\langle \v,\g\rangle-(\|\v\|^{2}+2\langle \v,\g\rangle+\frac{\langle \v,\g\rangle^{2}}{\|\v\|^{2}})}{\|\v\|^{2}+\|\g\|^{2}-2\langle \v,\g\rangle-(\|\v\|^{2}-2\langle \v,\g\rangle+\frac{\langle \v,\g\rangle^{2}}{\|\v\|^{2}})}=1.
\end{align*}
\end{proof}

\subsection{Distance between normalized matrices}
\begin{lemma}\label{lem:normstab_app}
Let $A,B\in\R^{n\times n}$ be nonzero and assume $\|A-B\|\le \rho <\frac12\|B\|$.
Then
\[
\left\|\frac{A}{\|A\|}-\frac{B}{\|B\|}\right\|\le \frac{4\rho}{\|B\|}.
\]
\end{lemma}

\begin{proof}
Write
\[
\frac{A}{\|A\|}-\frac{B}{\|B\|}
=
\frac{A-B}{\|A\|} + B\left(\frac{1}{\|A\|}-\frac{1}{\|B\|}\right).
\]
Since $\|A\|\ge \|B\|-\|A-B\|\ge \|B\|-\rho\ge \frac12\|B\|$, we have
$\|A\|^{-1}\le 2\|B\|^{-1}$. Also,
\[
\left|\frac{1}{\|A\|}-\frac{1}{\|B\|}\right|
=
\frac{|\|A\|-\|B\||}{\|A\|\|B\|}
\le
\frac{\rho}{(\frac12\|B\|)\|B\|}
=
\frac{2\rho}{\|B\|^2}.
\]
Therefore,
\[
\left\|\frac{A}{\|A\|}-\frac{B}{\|B\|}\right\|
\le
\frac{\|A-B\|}{\|A\|}
+
\|B\|\left|\frac{1}{\|A\|}-\frac{1}{\|B\|}\right|
\le
\frac{\rho}{\frac12\|B\|}
+
\|B\|\frac{2\rho}{\|B\|^2}
=
\frac{4\rho}{\|B\|}.
\]
\end{proof}

\subsection{Column norm bound implies spectral norm bound}
\begin{lemma}\label{lem:col2spec_app}
Let $E\in\R^{n\times n}$. If $\|E\e_i\|\le \varepsilon_c$ for all
$i=1,\dots,n$, then
\[
\|E\|\le \varepsilon_c\sqrt{n}.
\]
\end{lemma}

\begin{proof}
For any $\v\in\R^n$ with $\|\v\|=1$,
\[
E\v=\sum_{i=1}^n v_i(E\e_i),
\]
so
\[
\|E\v\|
\le \sum_{i=1}^n |v_i|\|E\e_i\|
\le \varepsilon_c\sum_{i=1}^n |v_i|
\le \varepsilon_c\sqrt{n}\Big(\sum_{i=1}^n v_i^2\Big)^{1/2}
=\varepsilon_c\sqrt{n}.
\]
Taking the supremum over $\|\v\|=1$ yields $\|E\|\le \varepsilon_c\sqrt{n}$.
\end{proof}

\subsection{Inner Product Concentration for Random Vectors on Sphere}\label{append:inner-product-append}
Here we give a lemma of the inner product concentration for random vectors on sphere proved in \cite{tao2026comparison}, stated below: 
\begin{lemma}[Lemma~3 of \cite{tao2026comparison}]\label{lem:coordinate-bounds}
Let $n\ge 5$. For any $\x\in \mathbb{R}^n$, $\x \ne \mathbf{0}$, and any constant $c>0$, there exists constant $p_1$ and $p_2$ which is independent of $n$, such that
\[
p_1 \le \Pr_{\y\sim S_n}[|\langle \y, \x\rangle| \le \|\x\|/(c\sqrt{n})]\le p_2.
\]
where $\y$ is chosen from $S_n$ uniformly at random. In particular, we have the inequality
\begin{equation}
\Pr_{\y\sim S_n}\left[|\langle \y, \x\rangle| \le \|\x\|\cdot\frac{24}{25\sqrt{n}}\right]\ge 3/5.
\end{equation}
\end{lemma}


\section{Basic Lemmas for Robust Hessian Estimation}\label{append:proofs}
In this section, we provide some supplementary lemmas for the guaranties claimed in \sec{hess_est}. Throughout, we assume that $\|H\e_1\|=\max_i\|H\e_i\|>0$. We can make this assumption since we can find a column with maximum norm after obtaining all the column norm ratios, and it suffices to denote this column as index $1$.
\subsection{Notation}
Let $H\in\R^{n\times n}$ be symmetric, $H\neq 0$, with columns
$H=(\h_1,\dots,\h_n)$, $\h_i:=He_i$. Define
\[
r_i:=\frac{\|\h_i\|}{\|\h_1\|}\in[0,1],\qquad r_1=1.
\]
When $\h_i\neq 0$, define the unit column direction $u_i:=\h_i/\|\h_i\|$.
When $\h_1+\h_i\neq 0$, define the unit sum direction
$u_{1i}:=(\h_1+\h_i)/\|\h_1+\h_i\|$. Define the scale-free target matrix
\begin{equation}\label{eq:Hstar_app}
H^\star := [r_1\u_1,\dots,r_n\u_n]\in\R^{n\times n}.
\end{equation}
Then $H=\|\h_1\|H^\star$ and hence $H/\|H\|=H^\star/\|H^\star\|$.

\begin{lemma}\label{lem:Hstar_lb_app}
The matrix $H^\star$ be in \eq{Hstar_app} satisfies $\|H^\star\|\ge 1$.
\end{lemma}

\begin{proof}
Because $r_1=1$ and $\u_1$ is a unit vector, the first column of $H^\star$ equals
$H^\star \e_1 = \u_1$ and thus
\[
\|H^\star\| \ge \|H^\star \e_1\| = \|\u_1\| = 1.
\]
\end{proof}

\subsection{Error bound of column vector estimation}
Here we give some lemmas for proving the fact that the small perturbation $\h_i+\rho\h_j$ can only affect the estimation of $\h_i$ under a bounded error when $\rho$ is bounded.

\begin{lemma}\label{lem:suit_tri}
Our choice of $\g_i$ and $\g_{1i}$ in \algo{normhess}, \lin{line7}, \lin{line10}, \lin{line14} and \lin{line17} satisfies $\|\g_i-\u_i\|\le \eta$ and $\|\g_{1i}-\u_{1i}\|\le \eta$.
\end{lemma}

\begin{proof}
We analyze $\g_i$ here, $\g_{1i}$ case is similar. Assume that $\v$ is an eigenvector of $H$ corresponding to minimum eigenvalue, and $|\langle \e_t, \v\rangle|\ge\frac{1}{\sqrt{n}}$. Fix an index $i$. Denote 
\[
\y_i^0=\e_i,\; \y_i^1=\e_i+\sigma \e_t,\; \y_i^2=\e_i-\sigma \e_t.
\]
Observe that for any $j\neq k\in\{0,1,2\}$,
\[
|\langle \y_i^j-\y_i^k, \v\rangle|\ge\frac{\sigma}{\sqrt{n}}.
\]
This implies at least two of three queries satisfies
\begin{equation}\label{eq:lower_bound_hv}
|\langle \y_i^j, \v\rangle|\ge\frac{\sigma}{2\sqrt{n}}.
\end{equation}
For each index $j$ satisfies \eq{lower_bound_hv}, by taking $\gamma_\y = \frac{\sigma}{2\sqrt{n}}$, $\gamma_\x=\epsilon$, accuracy $\hat \delta=\frac{\eta}{4}$ in \algo{Comparison-Triangle}, we know at least two queries in the setting $\{\y_i^0, \y_i^1,\y_i^2\}$ satisfies the condition $|\<\y_i^j, \v\>|\ge\gamma_\y$. Therefore, according to \thm{triangle}, at least two of three inequalities hold:
\[
\left\| \g_i^0-\frac{\h_i}{\|\h_i\|}\right\|\le\frac{\eta}{4},\quad \left\| \g_i^+-\frac{\h_i+\sigma \h_t}{\|\h_i+\sigma \h_t\|}\right\|\le\frac{\eta}{4}, \quad \left\| \g_i^--\frac{\h_i-\sigma \h_t}{\|\h_i-\sigma \h_t\|}\right\|\le\frac{\eta}{4}.
\]
In non-degenerate case, our choice $(\v_1, \v_2)\gets \arg\max_{\v_1\neq \v_2\in\{\g_i^{0},\g_i^{+},\g_i^{-}\}}\langle \v_1, \v_2\rangle$ tells $\v_1$ and $\v_2$ are two success queries to \algo{Comparison-Triangle}, take $\v_1=\g_i^+$, $\v_2=\g_i^-$ for example, other case is the same, from \lem{dist-norm-vector} we have
\begin{align*}
\left\|\frac{\v_1+\v_2}{\|\v_1+\v_2\|}-\frac{\h_i}{\|\h_i\|}\right\|
&\le\left\| \g_i^+-\frac{\h_i+\sigma \h_t}{\|\h_i+\sigma \h_t\|}\right\|+\left\| \g_i^--\frac{\h_i-\sigma \h_t}{\|\h_i-\sigma \h_t\|}\right\|+\left\| \frac{\h_i+\sigma \h_t}{\|\h_i+\sigma \h_t\|}-\frac{\h_i-\sigma \h_t}{\|\h_i-\sigma \h_t\|}\right\| \\
&\le \frac{\eta}{4}+\frac{\eta}{4}+\frac{2\eta}{4}=\eta,
\end{align*}
where the last inequality holds because $\langle \g_i^{+},\g_i^{-}\rangle \ge 1-\tau_\alpha$ tell us $\|\h_i\|/\|\sigma \h_t\|\ge\frac{4}{\eta}$.

In the degenerate case $\langle \g_i^{+},\g_i^{-}\rangle \le 1-\tau_\alpha$, we have $\|\h_i\|/\|\sigma \h_t\|\le\frac{4}{\eta}$, i.e.,
\[
\|\h_i\|\le\frac{\eta}{32\sqrt{n}}.
\]
In this case, taking $\g_i=0$ guarantees error on this column is $O\left(\frac{\eta}{\sqrt{n}}\right)$.
\end{proof}

The following lemma is to bound the error between the normalized vector and the perturbed one.

\begin{lemma}\label{lem:bias_app}
Let $\ensuremath{\mathbf{a}},\ensuremath{\mathbf{c}}\in\R^n$ with $\ensuremath{\mathbf{a}}\neq 0$. For $\rho\in(0,1/2)$ define
\[
\v := \frac{\ensuremath{\mathbf{a}}+\rho \ensuremath{\mathbf{c}}}{\|\ensuremath{\mathbf{a}}+\rho \ensuremath{\mathbf{c}}\|},\qquad \u:=\frac{\ensuremath{\mathbf{a}}}{\|\ensuremath{\mathbf{a}}\|}.
\]
If $\|\ensuremath{\mathbf{c}}\|\le \|\h_1\|$ and $\|\ensuremath{\mathbf{a}}\|\ge r_{\min}\|\h_1\|$, then
\[
\|\v-\u\|\le \frac{4}{r_{\min}}\,\rho.
\]
\end{lemma}

\begin{proof}
By triangle inequality, $\|\ensuremath{\mathbf{a}}+\rho \ensuremath{\mathbf{c}}\|\ge \|\ensuremath{\mathbf{a}}\|-\rho\|\ensuremath{\mathbf{c}}\|\ge \|\ensuremath{\mathbf{a}}\|-\rho\|\h_1\|$.
Using $\|\ensuremath{\mathbf{a}}\|\ge r_{\min}\|\h_1\|$ and $\rho\le r_{\min}/2$, we get
$\|\ensuremath{\mathbf{a}}+\rho \ensuremath{\mathbf{c}}\|\ge \frac12\|\ensuremath{\mathbf{a}}\|$. Therefore,
\[
\v-\u
=
\frac{\ensuremath{\mathbf{a}}+\rho \ensuremath{\mathbf{c}}}{\|\ensuremath{\mathbf{a}}+\rho \ensuremath{\mathbf{c}}\|}-\frac{\ensuremath{\mathbf{a}}}{\|\ensuremath{\mathbf{a}}\|}
=
\ensuremath{\mathbf{a}}\left(\frac{1}{\|\ensuremath{\mathbf{a}}+\rho \ensuremath{\mathbf{c}}\|}-\frac{1}{\|\ensuremath{\mathbf{a}}\|}\right)+\frac{\rho}{\|\ensuremath{\mathbf{a}}+\rho \ensuremath{\mathbf{c}}\|}\ensuremath{\mathbf{c}}.
\]
Hence
\[
\|\v-\u\|
\le
\|\ensuremath{\mathbf{a}}\|\left|\frac{1}{\|\ensuremath{\mathbf{a}}+\rho \ensuremath{\mathbf{c}}\|}-\frac{1}{\|\ensuremath{\mathbf{a}}\|}\right| + \frac{\rho}{\|\ensuremath{\mathbf{a}}+\rho \ensuremath{\mathbf{c}}\|}\|\ensuremath{\mathbf{c}}\|.
\]
Also
\[
\left|\frac{1}{\|\ensuremath{\mathbf{a}}+\rho \ensuremath{\mathbf{c}}\|}-\frac{1}{\|\ensuremath{\mathbf{a}}\|}\right|
=
\frac{|\|\ensuremath{\mathbf{a}}+\rho \ensuremath{\mathbf{c}}\|-\|\ensuremath{\mathbf{a}}\||}{\|\ensuremath{\mathbf{a}}+\rho \ensuremath{\mathbf{c}}\|\|\ensuremath{\mathbf{a}}\|}
\le
\frac{\rho\|\ensuremath{\mathbf{c}}\|}{\|\ensuremath{\mathbf{a}}+\rho \ensuremath{\mathbf{c}}\|\|\ensuremath{\mathbf{a}}\|},
\]
by which we can conclude that
\[
\|\v-\u\|
\le
\|\ensuremath{\mathbf{a}}\|\frac{\rho\|\ensuremath{\mathbf{c}}\|}{\|\ensuremath{\mathbf{a}}+\rho \ensuremath{\mathbf{c}}\|\|\ensuremath{\mathbf{a}}\|} + \frac{\rho}{\|\ensuremath{\mathbf{a}}+\rho \ensuremath{\mathbf{c}}\|}\|\ensuremath{\mathbf{c}}\|
=
\frac{2\rho\|\ensuremath{\mathbf{c}}\|}{\|\ensuremath{\mathbf{a}}+\rho \ensuremath{\mathbf{c}}\|}
\le
\frac{2\rho\|\h_1\|}{\frac12\|\ensuremath{\mathbf{a}}\|}
=
\frac{4\rho}{r_{\min}}.
\]
\end{proof}


\section{Basic Lemmas for Estimating the Hessian-to-Gradient Ratio}
\label{append:ratio_proofs}

We provide some supplementary lemmas of the claims in \sec{grad_hess_ratio} in this appendix. We first introduce the core lemma in our proof:

\begin{lemma}[Davis-Kahan Theorem, see e.g., Theorem 1 of~\cite{yu2015daviskahan}]\label{lem:Davis-Kahan}
Let $A, \widetilde A\in \R^{d\times d}$ be two symmetric matrices satisfying $\|A-\widetilde A\|\le\xi$ for some $\xi>0$. For any $a<b$, denote $S=\{v_1,\ldots,v_k\}$ and $\widetilde S=\{\tilde v_1,\ldots,\tilde v_k\}$ as the set of normalized eigenvectors of $A$ and $\widetilde A$ associated with eigenvalues contained in the interval $[a,b]$ and $[a-\xi,b+\xi]$ respectively, and denote 
\[
V:=\text{span}(S), \qquad \widetilde V:=\text{span}(\widetilde S).
\]
Then, if the remaining eigenvalues of $A$ lie outside the interval $[a-\gamma,b+\gamma]$, we have $k=\tilde k$ and 
\[
\left\|\sin(\Theta(V,\widetilde V))\right\|\le\frac{\xi}{\gamma},
\]
where
\[
\sin(\Theta(V,\widetilde V)):=\diag(\sin\theta_1(V,\widetilde V),\ldots,\sin\theta_k(V,\widetilde V))^\top.
\]
\end{lemma}

\begin{lemma}[Weyl's Theorem, see \cite{Weyl1909berBQ}]\label{lem:Weyl}
Let $A,B$ be Hermitian on inner product space V with dimension n, with spectrum ordered in descending order $\lambda_1\ge\ldots\ge\lambda_n$, and let $1\le i,j\le n$ be integers. If $i+j\le n+1$, then 
\[
\lambda_{i+j-1}(A+B)\le \lambda_i(A)+\lambda_j(B).
\] 
If $n<i+j$, then 
\[
\lambda_i(A)+\lambda_j(B)\le\lambda_{i+j-n}(A+B).
\]
\end{lemma}

\begin{corollary}[Spectral stability]
In the setting of \lem{Weyl}, let $1\le k\le n$ be an integer, then
\[
|\lambda_k(A+B)-\lambda_k(A)|\le\|B\|.
\]
\end{corollary}
We use the following conversion: if $\hat \z$ is a unit vector satisfying $\|\hat \z-\z/\|\z\|\|\le \xi$, then one may write
\begin{equation}\label{eq:noise_decomp}
\hat \z=\frac{\z+\e}{\|\z+\e\|}\quad\text{for some $\e$ with }\|\e\|\le \xi\|\z\|.
\end{equation}

\begin{lemma}
\label{lem:noise_decomp}
If $\z\neq 0$ and $\hat \z$ is unit with $\|\hat \z-\z/\|\z\|\|\le \xi<1$, then there
exists $\e$ with $\|\e\|\le 2\xi\|\z\|$ such that \eq{noise_decomp} holds.
\end{lemma}

\begin{proof}
Let $\bar \z=\z/\|\z\|$. Define $\e:=\|\z\|(\hat \z-\bar \z)$. Then $\|\e\|\le \xi\|\z\|$ and
$\z+\e=\|\z\|\hat \z$, hence $\hat \z=(\z+\e)/\|\z+\e\|$.
\end{proof}

Applying \lem{noise_decomp} to the two gradient direction queries:
\begin{equation}\label{eq:g_noise_forms}
\hat{\g}_1=\frac{\g+k_1 \v_1}{\|\g+k_1 \v_1\|},\qquad
\hat{\g}_2=\frac{\nabla f(\x+\mu \v)+k_2 \v_2}{\|\nabla f(\x+\mu \v)+k_2 \v_2\|},
\end{equation}
with $\|k_1 \v_1\|\le O(\xi_1)\|\g\|$ and
$\|k_2 \v_2\|\le O(\xi_2)\|\nabla f(\x+\mu \v)\|$.

\begin{lemma}
\label{lem:taylor_grad}
Given \assum{hess_lip}, any unit vector $\v$ satisfies
\[
\nabla f(\x+\mu \v)=\nabla f(\x)+\mu \nabla^2 f(\x)\v + \ensuremath{\mathbf{r}}_\mu
\quad\text{with}\quad
\|\ensuremath{\mathbf{r}}_\mu\|\le \tfrac{1}{2}L_2\mu^2.
\]
\end{lemma}

\begin{proof}
Standard integral remainder; included for completeness.
\[
\nabla f(\x+\mu \v)-\nabla f(\x)=\int_0^\mu \nabla^2 f(\x+t \v)\,\v\,\d t
=\mu \nabla^2 f(\x)\v+\int_0^\mu (\nabla^2 f(\x+t \v)-\nabla^2 f(\x))\v\,\d t,
\]
so $\|\ensuremath{\mathbf{r}}_\mu\|\le \int_0^\mu L_2 t\,\d t=\tfrac{1}{2}L_2\mu^2$.
\end{proof}

\begin{lemma}
\label{lem:eta_prime_enters}
Assume the colinearity condition of \lem{difference-colinear}. Let
\[
\mathbf{d}_1:=\g+k_1\v_1,\qquad
\mathbf{d}_2:=\nabla f(\x+\mu \v)+k_2\v_2,\qquad
\mathbf{d}:=\eta\,\mu\|H\| \widehat H\v.
\]
If $\frac{\mathbf{d}_2}{\|\mathbf{d}_2\|}= \frac{\mathbf{d}_1+\mathbf{d}}{\|\mathbf{d}_1+\mathbf{d}\|}$, then
\begin{equation}\label{eq:eta_prime_ratio_exact}
\frac{\|\mathbf{d}\|}{\|\mathbf{d}_1\|}
=
\frac{\sqrt{1-\inner{\hat{\g}_1}{\hat{\g}_2}^2}}{\sqrt{1-\inner{\hat{\g}_2}{\u}^2}}.
\end{equation}
\end{lemma}

\begin{proof}
Apply \lem{vector-norm-1} to the pair $(\v,\g)=(\mathbf{d}_1,\mathbf{d})$: it expresses the norm
ratio $\|\mathbf{d}\|/\|\mathbf{d}_1\|$ via the sines of the angles between $\mathbf{d}_1$, $\mathbf{d}_1+\mathbf{d}$ and $\mathbf{d}$.
Noting that $\hat{\g}_1=\mathbf{d}_1/\|\mathbf{d}_1\|$ and $\hat{\g}_2=\mathbf{d}_2/\|\mathbf{d}_2\|$ by
\eq{g_noise_forms}, we obtain \eq{eta_prime_ratio_exact}.
\end{proof}


\section{Basic Lemmas for Finding Stationary Points}\label{append:best_of_all_proof}
We give the proof details of \sec{best_of_all} in this appendix. First we refer to a lemma for Taylor approximation error analysis.

\begin{lemma}[Lemma~1 of \cite{article}]\label{lem:taylor_cubic_best}
Suppose that $f\colon\R^d\to\R$ has $L_2$-Lipschitz Hessian. Then for all $\x,\p\in\R^d$,
denote
\[
T_\x(\x+\p)=f(\x)+\inner{\nabla f(\x)}{\p}+\frac12 \p^\top \nabla^2 f(\x)\p,
\]
we have
\begin{equation}\label{eq:taylor_cubic_best}
|f(\x+\p)-T_\x(\x+\p)|\le \frac{L_2}{6}\|\p\|^3.
\end{equation}
\end{lemma}

We give a lemma for the descent guaranty of quadratic approximation.
\begin{lemma}\label{lem:tr_quadratic}
Given a quadratic $q(\p)=q(\0)+b^\top \p+\frac12 \p^\top A \p$ with symmetric $A$, and assume that $\|A\|\ge\sqrt{L_2\epsilon}$, and $A$ is not positive or negative definite. Denote the eigenvector of $A$ that corresponding to the largest magnitude eigenvalue as $\v$ and assume that $\inner{\v}{b}\le\frac12\|b\|$. For $r>0$ and
let $\p^\star\in\arg\min_{\|\p\|\le r} q(\p)$. Then either:
\begin{itemize}
\item (Interior) $\|\p^\star\|<r$ and $\nabla q(\p^\star)=b+A \p^\star=\0$;
\item (Boundary) $\|\p^\star\|=r$ and $q(\p^\star)-q(\0)\le -\frac12 r\|\nabla q(\p^\star)\|$.
\end{itemize}
\end{lemma}

\begin{proof}
If $\|\p^\star\|<r$, first-order
optimality gives $b+A \p^\star=\0$. Otherwise $\|\p^\star\|=r$ and there exists
$\lambda\ge 0$ with $b+A \p^\star+\lambda \p^\star=\0$ and $A+\lambda I\succeq \0$.
Then $\nabla q(\p^\star)=b+A \p^\star=-\lambda \p^\star$ so $\|\nabla q(\p^\star)\|=\lambda r$.
Moreover,
\[
q(\0)-q(\p^\star)=-b^\top \p^\star-\tfrac12 (\p^\star)^\top A \p^\star.
\]
From $b+A \p^\star+\lambda \p^\star=\0$ we get $-b^\top \p^\star=\lambda r^2+(\p^\star)^\top A \p^\star$.
Using $A+\lambda I\succeq \0$ we have $(\p^\star)^\top A \p^\star\ge -\lambda r^2$, hence
$q(\0)-q(\p^\star)\ge \frac12\lambda r^2=\frac12 r\|\nabla q(\p^\star)\|$.
\end{proof}

\section{Lazy Trust Region Method with Evaluation Oracle}\label{append:lazy_TR}
In this section, we discuss how to combine our approach and lazy-type methods preliminarily. Consider the following problem:
\problem{Given an objective function $f\colon\R^n\to\R$ that has $L_1$-Lipschitz gradient and $L_2$-Lipschitz Hessian. The query oracle of $f$ is zeroth-order, i.e., we are given the oracle $O_f\colon \x\mapsto f(\x)$. The goal is to output a list containing an FOSP.}

\begin{algorithm2e}[htbp!]
	\caption{LazyTrustRegion($\x_0$, $\epsilon$)}
	\label{algo:LazyTrustRegion}
	\LinesNumbered
	\DontPrintSemicolon
    \KwInput{Starting point $\x_0$, precision $\epsilon$}
    \KwOutput{an $\epsilon$-first order stationary point $\x_T$}
    \For{$k=0,1,\ldots,T$}{
    $\widetilde H_{km}=\text{HessianEstimation}(\x_{km}).$
    
    \For{$t=0,1.\ldots,m-1$}{
    $\g_{km+t}\gets\text{GradientEstimation}(\x_{km+t}).$

    $\x_{km+t+1}\gets\min_{\|\x-\x_{km+t}\|\le r}\left\{\<\nabla f(\x), \x-\x_{km+t}\>+\frac{1}{2}(\x-\x_{km+t})^\top \widetilde H_{km}(\x-\x_{km+t})\right\} $
    $\triangleq\min_{\|\x-\x_{km+t}\|\le r}\{m_{km+t}(\p)\}.$
    }
    }
    \Return{$\x_{T+1}$}
\end{algorithm2e}

\emph{In lazy-type methods, we update Hessian every $m$ steps. We guarantee that in these $m$ steps, the Hessian in the current point and delayed Hessian are close enough so that using the delayed Hessian introduces bearable error. According to this idea we obtain the following theorem:}

\begin{theorem}\label{thm:lazy_TR_value}
Take $m=n$. Then in $O(\sqrt{n}/\epsilon^{1.5}+n)$ iteration steps, with high probability, we can visit an $\epsilon$-first order stationary point.
\end{theorem}

\begin{proof}[Proof of \thm{lazy_TR_value}]
Take stepsize $r=c\sqrt{\epsilon/n}$ where $c$ is an absolute constant that is small enough, and let $\x_{km+t+1}=\x_{km+t}+\p$. By \lem{taylor_cubic_best}, we have 
\begin{equation}
|f(\x_{km+t+1})-T_{\x_{km+t}}(\p)|\le \frac{L_2}{6}\left(\frac{\epsilon}{n}\right)^{\frac{3}{2}}.
\end{equation}
Due to the $L_2$-Lipschitzness of Hessian, after $t\le m$ steps, we have
\begin{align}
\|\hat H_{km+t}-\widetilde H_{km}\|&\le L_2\|\x_{km+t}-\x_{km}\|\le L_2\sum_{i=km}^{km+t-1}\|\x_{i+1}-\x_{i}\|\le L_2\cdot rm.
\end{align}
When $\|\x_{t+1}-\x_t\|\le r$, the error, denoted $\Delta_{\text{err}}$, can be estimated as
\begin{align}
\Delta_{\text{err}}:=|T_{\x_{km+t}}(\p)-m_{km+t}(\p)|&=\frac{1}{2}(\x-\x_{km+t})^\top(\hat H_{km+t}-\widetilde H_{km})(\x-\x_{km+t})\notag\\
&\le\frac{1}{2}r^2 L_2\cdot rm
=\frac{L_2}{2}r^3m
=O(r^3\cdot m).
\end{align}
From \lem{tr_quadratic}, we have
\begin{align}
|T_{\x_{km+t}}(\p)-f(\x_{km+t})|&\ge \frac{r}{2}\|\nabla f(\x_{km+t})+\hat H_{km+t}\cdot \p\|=O(r\epsilon).
\end{align}
Choosing $c$ small enough for $r=c\sqrt{\epsilon/n}$, we obtain $|\Delta_{km+t+1}-\Delta_{km+t}|\ge2\Delta_{\text{err}}$, thus we can get a descent at least $\Theta(\epsilon^{1.5}/\sqrt{n})$. When $\|\x_{t+1}-\x_t\|<r$, we have
\begin{align}
\nabla f(\x_{km+t+1})&=\nabla f(\x_{km+t})+\widetilde H_{km}(\x_{km+t+1}-\x_{km+t})+(\nabla^2f(\x_{km+t})-\widetilde H_{km})(\x_{km+t+1}-\x_{km+t})\notag\\
&=(\nabla^2f(\x_{km+t})-\widetilde H_{km})(\x_{km+t+1}-\x_{km+t}).
\end{align}
As a result, we can obtain
\begin{align}
\|\nabla f(\x_{km+t+1})\|&\le\|\nabla^2f(\x_{km+t})-\widetilde H_{km}\|\cdot\|\x_{km+t+1}-\x_{km+t}\|\le rd\cdot r\le\epsilon.
\end{align}
In \algo{LazyTrustRegion}, the total times that we update the gradient is $T_{\text{iter}}=O(\sqrt{n}/\epsilon^{1.5})$. Therefore, the total query complexity is
\begin{align}
n\cdot \left(T_{\text{iter}}+\left(T_{\text{iter}}\cdot\frac{1}{m}+1\right)n\right)=O\left(\frac{n^{1.5}}{\epsilon^{1.5}}+n^2\right).
\end{align}
\end{proof}

\end{document}